\documentclass[11pt]{scrartcl}
\usepackage{lineno,verbatim}
\usepackage{mathtools}
\mathtoolsset{showonlyrefs}
\usepackage{amssymb,amsmath,amsthm}
\usepackage{mathrsfs}  
\usepackage{epsfig}
\usepackage{tikz}
\usepackage{graphicx}
\usepackage{graphics}
\usepackage{float}
\usepackage{multirow}
\usepackage{color}
\usepackage{lineno}
\usepackage{fullpage}
\usepackage[normalem]{ulem} 
\usepackage{makeidx}
\usepackage{xspace}
\usepackage{wrapfig}
\usepackage{todonotes}
\usepackage{enumerate}
\usepackage{hyperref}
\usepackage{algorithm}
\usepackage{subcaption}
\usepackage{afterpage}
\usepackage[noend]{algpseudocode}
\usepackage[symbol]{footmisc}

\theoremstyle{definition}

\newtheorem{Example}{Example}
\newtheorem{Remark}{Remark}

\theoremstyle{plain}
\newtheorem{theorem}{Theorem}

\newtheorem{corollary}{Corollary}

\newtheorem{Proposition}{Proposition}

\definecolor{darkred}{rgb}{1, 0.1, 0.3}
\definecolor{darkblue}{rgb}{0.1, 0.1, 1}
\definecolor{darkgreen}{rgb}{0,0.6,0.5}

\newcommand {\mm}[1] {\ifmmode{#1}\else{\mbox{\(#1\)}}\fi}

\newcommand{\eps}		{\varepsilon}
\newcommand{\dx}{\,\mathrm{d}}

\newcommand{\OT}{\mathrm{OT}}
\newcommand{\INN}{\mathrm{INN}}

\newcommand{\Lip}{\operatorname{Lip}}

\newcommand{\sigA}{\mathfrak{A}}

\newcommand{\R}{\mathbb{R}}
\newcommand{\N}{\ensuremath{\mathbb{N}}}
\newcommand{\E}{\mathbb{E}}

\DeclareMathOperator{\Cov}{Cov}
\newcommand{\X}{{\R^m}}
\newcommand{\Y}{{\R^n}}
\newcommand{\supp}{\textnormal{supp}}
\DeclareMathOperator{\dist}{dist}
\DeclareMathOperator*{\argmin}{argmin}
\DeclareMathOperator*{\argmax}{argmax}
\DeclareMathOperator*{\softmax}{softmax}
\DeclareMathOperator*{\diam}{diam}
\DeclareMathOperator*{\pad}{pad}
\DeclareMathOperator*{\data}{data}
\newcommand{\weakly}{\rightharpoonup}

\newcommand{\reg}{\text{reg}}

\begin{document}

\title{Stabilizing Invertible Neural Networks\\ Using Mixture Models}
 
\author{Paul Hagemann\footnotemark[1]  \footnotemark[2] \and Sebastian Neumayer\footnotemark[1]}
\maketitle
\footnotetext[1]{Institute of Mathematics,
	TU Berlin,
	Strasse des 17.~Juni 136, 10623 Berlin, Germany,
	neumayer@math.tu-berlin.de}
\footnotetext[2]{Department of Mathematical Modelling and Data Analysis, Physikalisch-Technische Bundesanstalt, Abbestr. 2-12, 10587 Berlin, Germany}
\begin{abstract}
    In this paper, we analyze the properties of invertible neural networks, which provide a way of solving inverse problems.
    Our main focus lies on investigating and controlling the Lipschitz constants of the corresponding inverse networks.
    Without such a control, numerical simulations are prone to errors and not much is gained against traditional approaches.
    Fortunately, our analysis indicates that changing the latent distribution from a standard normal one to a Gaussian mixture model resolves the issue of exploding Lipschitz constants.
    Indeed, numerical simulations confirm that this modification leads to significantly improved sampling quality in multimodal applications.
\end{abstract}
\section{Introduction}
Reconstructing parameters of physical models is an important task in science.
Usually, such problems are severely underdetermined and sophisticated reconstruction techniques are necessary.
Whereas classical regularisation methods focus on finding just the most desirable or most likely solution of an inverse problem, more recent methods focus on analyzing the complete distribution of possible parameters.
In particular, this provides us with a way to quantify how trustworthy the obtained solution is.
Among the most popular methods for uncertainty quantification are Bayesian methods \cite{gamerman2006markov}, which build up on evaluating the posterior using Bayes theorem, and Markov Chain Monte Carlo (MCMC) \cite{roberts2004}.
Over the last years, these classical frameworks and questions have been combined with neural networks (NNs).
To give a few examples, the Bayesian framework has been extended to NNs \cite{AO19, LW17}, auto-encoders have been combined with MCMC \cite{BYAV13} and generative adversarial networks have been applied for sampling from conditional distributions \cite{mirza2014conditional}.
Another approach in this direction relies on so-called invertible neural networks (INNs) \cite{ardizzone2018analyzing, inv_resnet, DBLP:conf/iclr/DinhSB17, glow} for sampling from the conditional distribution.
Such models have been successfully applied for stellar parameter estimation \cite{ksoll2020stellar}, conditional image generation \cite{ardizzone2020conditional} and polyphonic music transcription \cite{SlicedWasserstINN}.

Our main goal in this article is to deepen the mathematical understanding of INNs.
Let us briefly describe the general framework.
Assume that the random vectors $(X,Y)\colon \Omega\to\R^d\times\R^n$ are related via $f(X) \approx Y$ for some forward mapping $f\colon \R^d\to\R^n$.
Clearly, this setting includes the case of noisy measurements.
In practice, we often only have access to the values of $Y$ and want to solve the corresponding inverse problem.
To this end, our aim is to find an invertible mapping $T\colon\R^d\to\R^n\times \R^k$ such that
\begin{itemize}
    \item $T_y(X) \approx Y$, i.e., that $T_y$ is an approximation of $f$; 
    \item $T^{-1}(y,Z)$ with $Z\sim\mathcal N(0, I_k)$, $k=d-n$, is distributed according to $P_{(X|Y)}(y,\cdot)$.
\end{itemize}
Throughout this paper, the mapping $T$ is modelled as an INN, i.e., as NN with an explicitly invertible structure.
Therefore, no approximations are necessary and evaluating the backward pass is as simple as the forward pass.
Here, $Z$ should account for the information loss when evaluating the forward model $f$.
In other words, $Z$ parameterizes the ill-posed part of the forward problem $f$.

Our theoretical contribution is twofold.
First, we generalize the consistency result for perfectly trained INNs given in \cite{ardizzone2018analyzing} towards latent distributions $Z$ depending on the actual measurement $y$.
Here, we want to emphasize that our proof does not require any regularity properties for the distributions.
Next, we analyze the behavior of $\Lip(T^{-1})$ for a homeomorphism $T$ that approximately maps a multimodal distribution to the standard normal distribution.
More precisely, we obtain that $\Lip(T^{-1})$ explodes as the amount of mass between the modes and the approximation error decreases.
In particular, this implies for INNs that using $Z \sim N(0,I_k)$ can lead to instabilities of the inverse map for multimodal problems.
As the results are formulated in a very general setting, they could be also of interest for generative models, which include variational auto-encoders \cite{Kingma_2019} and generative adversarial networks (GANs) \cite{GANs}.
Another interesting work that analyzes the Lipschitz continuity of different INN architectures is \cite{behrmann2020understanding}, where also the effects of numerical non-invertibility are illustrated.

The performed numerical experiments confirm our theoretical findings, i.e., that $Z\sim\mathcal N(0, I_k)$ is unsuitable for recovering multimodal distributions.
Note that this is indeed a well-known problem in generative modelling \cite{huang2018neural}.
Consequently, we propose to replace $Z \sim \mathcal N(0, I_k)$ by a Gaussian mixture model with learnable probabilities that depend on the measurements $y$.
In some sense, this enables the model to decide how many modes it wants to use in the latent space.
Therefore, splitting or merging distributions should not be necessary anymore.
Such an approach has lead to promising results for GANs \cite{disc_mani, DeLiGAN17} and variational auto-encoders \cite{VAE_CLUSTER}.
In order to learn the probabilities corresponding to the modes, we employ the Gumbel trick \cite{jang2016categorical, maddison2016concrete} and a $L^2$-penalty for the networks weights, enforcing the Lipschitz continuity of the INN.
More precisely, the $L^2$-penalty ensures that splitting or merging distributions is unattractive.
Similar approaches have been successfully applied for generative modelling, see \cite{VAE_CLUSTER, dinh2019rad}.

\paragraph{Outline}
We recall the necessary theoretical foundations in Section~\ref{sec:Pre}.
Then, we introduce the continuous optimization problem in Section~\ref{sec:INN} and discuss its theoretical properties.
In particular, we investigate the effect of different latent space distributions on the Lipschitz constants of an INN.
The training procedure for our proposed model is discussed in Section~\ref{sec:NumImp}.
For minimizing over the probabilities of the Gaussian mixture model, we employ the Gumbel trick.
In Section~\ref{sec:Examples}, we demonstrate the efficiency of our model on three inverse problems that exhibit a multimodal structure.
Conclusions are drawn in Section~\ref{sec:Conclusions}.

\section{Preliminaries}\label{sec:Pre}
\paragraph{Probability Theory and Conditional Expectations}

First, we give a short introduction to the required probability theoretic foundations based on \cite{AFP2000,B1996,FL2007,K2008}.
Let $(\Omega,\sigA,P)$ be a probability space and $\X$ be equipped with the Borel $\sigma$-algebra $\mathcal B(\X)$.
A \emph{random vector} is a measurable mapping $X\colon\Omega\to\X$ with \emph{distribution} $P_X \in \mathcal P(\X)$ defined by $P_X(A)=P(X^{-1}(A))$ for $A\in\mathcal B(\X)$.
For any random vector $X$, there exists a smallest $\sigma$-algebra on $\Omega$ such that $X$ is measurable, which is denoted with $\sigma(X)$.
Two random vectors $X$ and $Y$ are called equal in distribution if $P_X=P_Y$ and we write \smash{$X \overset{d}{=} Y$}.
The space $\mathcal P(\X)$ is equipped with the \emph{weak convergence},
i.e., a sequence $(\mu_k )_k \subset \mathcal P(\X)$
converges \emph{weakly} to $\mu$, written $\mu_k \weakly \mu$, if
\begin{equation}
\lim_{k \to \infty} \int_{\X} \varphi \dx \mu_k = \int_{\X} \varphi \dx \mu \quad \text{for all } \varphi \in \mathcal  C_b(\X),
\end{equation} 
where $\mathcal  C_b(\X)$ denotes the space of continuous and bounded functions.
The \emph{support of a probability measure} $\mu$ is the closed set 
\[\supp(\mu) \coloneqq \bigl\{ x \in \X: B \subset \X \text{ open, }x \in B  \implies \mu(B) >0\bigr\}.\]
In case that there exists $\rho_\mu \colon\X\to\R$ with $\mu(A)=\int_A \rho_\mu \dx x$ for all $A\in \mathcal B(\X)$, this $\rho_\mu$ is called the \emph{probability density} of $\mu$.
For $\mu \in \mathcal P(\X)$ and $1\leq p\leq \infty$, let $L^p(\X)$ be the Banach space (of equivalence classes)
of complex-valued functions with norm
\[\|f\|_{L^p(\X)} = \left( \int_\X |f(x)|^p \dx x \right)^\frac1p < \infty.\]
In a similar manner, we denote the Banach space of (equivalence classes of) random vectors $X$ satisfying $\E(|X|^p)<\infty$ by $L^p(\Omega,\sigA,P)$.
The \emph{mean} of a random vector $X$ is defined as
$\E(X)=(\E(X_1),...,\E(X_m))^T$.
Further, for any $X\in L^2(\Omega,\sigA,P)$ 
the \emph{covariance matrix} of $X$ is defined by $\Cov(X)=(\Cov(X_i,X_j))_{i,j=1}^m$, where $\Cov(X_i,X_j)=\E((X_i-\E(X_i))(X_j-\E(X_j)))$.

For a measurable function $T\colon \X \to \Y $, the \emph{push-forward} measure of $\mu \in \mathcal P(\X)$ by $T$ is defined as $T_\# \mu  \coloneqq \mu \circ T^{-1}$.
A measurable function $f$ on $\Y$ is integrable with respect to  $\nu \coloneqq T_\# \mu$ 
if and only if the composition $f \circ T$ is integrable with respect $\mu$, i.e.,
\begin{align}\label{eq:push_f}
	\int_{\Y} f (y) \dx \nu (y) &= \int_{\X} f(T(x)) \dx \mu(x).
\end{align}
The \emph{product measure} $\mu \otimes \nu \in \mathcal P(\X \times \Y)$ of $\mu \in \mathcal P (\X)$ and $\nu \in \mathcal P(\Y)$ is the unique measure satisfying $\mu \otimes \nu(A \times B) = \mu(A) \nu(B)$ for all $A \in \mathcal B(\X)$ and $B \in \mathcal B(\Y)$.

If $X\in L^1(\Omega,\sigA,P)$ and $Y\colon\Omega\to\Y$ is a second random vector, then there exists a unique $\sigma(Y)$-measurable random vector $\E(X|Y)$ called \emph{conditional expectation} of $X$ given $Y$ with the property that
\[
\int_A X \dx P = \int_A \E(X|Y)\dx P \quad \text{for all }A\in\sigma(Y).
\]
%If additionally $X\in L^2(\Omega,\sigA,P)$, then $\E(X|\sigG)$ can be characterized as the orthogonal projection of $X$ on $L^2(\Omega,\sigG,P)$, i.e., for any $Y\in L^2(\Omega,\sigG,P)$ it holds
%$$\int_\Omega (X-Y)^2\dx P\geq \int_\Omega (X-\E(X|\sigG))^2\dx P.$$
Further, there exists a $P_Y$-almost surely unique, measurable mapping $g\colon \Y\to\X$ such that $\E(X|Y)=g\circ Y$.
Using this mapping $g$, we define the \emph{conditional expectation} of $X$ given $Y=y$ as $\E(X|Y=y)=g(y)$.
%If additionally $X\in L^2(\Omega,\sigA,P)$, then every measurable mapping $\varphi\colon\Y\to\X$ satisfies
%$$\int_{\Omega}(X-\E(X|Y))^2 \dx P\leq \int_\Omega (X-\varphi\circ Y)^2 \dx P$$
%with equality if and only if $\varphi=\E(X|Y=\cdot)$ $P_Y$-almost surely.
For $A\in\sigA$ the \emph{conditional probability} of $A$ given $Y=y$ is defined by $P(A|Y=y)=\E(1_A|Y=y)$.
%Note that this very general definition of conditional probabilities leads to the following special cases:
%\begin{enumerate}[(i)]
%\item Let $A,B\in\sigA$ with $P(B)>0$. Then it holds
%$$P(A|B)\coloneqq P(A|1_B=1)=\E(1_A|1_B)(\omega)=\frac{P(A\cap B)}{P(B)}.$$
%\item If $X\colon\Omega\to\X$, $Y\colon\Omega\to\Y$ are discrete random vectors and $y\in\Y$ with %$P(Y=y)>0$, then for all $x\in\X$ the \emph{conditional distribution} of $X$ given $Y=y$ reads
%$$
%P_{(X|Y=y)}(x)\coloneqq P(1_{\{x\}}|Y=y)=\frac{P(X=x,Y=y)}{P(Y=y)}.
%$$
%\item Let $X\colon\Omega\to\X$ and $Y\colon\Omega\to\Y$ be random vectors such that the density functions $f_X$, $f_Y$ and $f_{X,Y}$ exist. 
%Then, the \emph{conditional distribution} $P_{(X|Y=y)}\coloneqq P(X\in \cdot |Y=y)$ of $X$ given %$Y=y$ is a probability measure on $\X$ with density
%$$
%f_{(X|Y=y)}(x)=\frac{f_{X,Y}(x,y)}{f_Y(y)}.
%$$
%\end{enumerate}
More generally, \emph{regular conditional distributions} are defined in terms of Markov kernels, i.e., there exists a mapping $P_{(X|Y)} \colon \Y \times \mathcal B(\X) \to \R$ such that
\begin{enumerate}[(i)]
    \item $P_{(X|Y)}(y,\cdot)$ is a probability measure on $ \mathcal B(\X)$ for every $y \in \Y$,
    \item $P_{(X|Y)}(\cdot,A)$ is measurable for every $A \in \mathcal B(\X)$ and for all $A \in \mathcal B(\X)$ it holds
    \begin{equation}\label{def:RegCondDist}
        P_{(X|Y)}(\cdot,A) = E(1_A \circ X| Y =\,\cdot) \qquad \text{$P_Y$-a.e.}
    \end{equation}
\end{enumerate}
For any two such kernels $P_{(X|Y)}$ and $Q_{(X|Y)}$ there exists $B \in \mathcal B(\Y)$ with $P_Y(B) = 1$ such that for all $y\in B$ and $A \in \mathcal B(\X)$ it holds $P_{(X|Y)}(y,A) =Q_{(X|Y)}(y,A)$.
Finally, let us conclude with two useful properties, which we need throughout this paper.
\begin{enumerate}[(i)]
    \item For random vectors $X\colon\Omega\to\X$ and $Y\colon\Omega\to\Y$ and measurable $f \colon \Y\to \X$ with $f(Y)X \in L^1(\Omega,\sigA,P)$ it holds $E(f(Y)X | Y) = f(Y) E(X | Y)$, where the products are meant component-wise.
    In particular, this implies for $y \in \Y$ that
    \begin{equation}\label{eq:PullOut}
        E(f(Y)X | Y=y) = f(y) E(X | Y=y).
    \end{equation}
    \item For random vectors $X\colon\Omega\to\X$, $Y\colon\Omega\to\Y$ and a measurable map $T \colon \X\to \R^p$ it holds
    \begin{equation}\label{eq:RegPush}
        P_{(X | Y)}\bigl(\cdot, T^{-1}(\cdot)\bigr) = P_{(T \circ X|Y)}.
    \end{equation}
\end{enumerate}

\paragraph{Discrepancies}\label{sec:Discr}
Next, we introduce a way to quantify the distance between two probability measures.
To this end, choose a measurable, symmetric and bounded function $K\colon \X \times \X \rightarrow \mathbb R$ that is integrally strictly positive definite, i.e., for any $f \neq 0 \in L^2(\X$)
it holds
\[
\iint\limits_{\X \times \X} K(x,y) f(x) f(y) \dx x \dx y > 0.
\]
Note that this property is indeed stronger than $K$ being strictly positive definite.

Then, the corresponding \emph{discrepancy} $\mathscr{D}_K(\mu,\nu)$ is defined via
\begin{align} \label{mercer_1}
\mathscr{D}^2_K(\mu,\nu)
=& \iint\limits_{\X\times\X} K \dx\mu \dx\mu - 2\iint\limits_{\X\times\X} K\dx\mu \dx\nu
+\iint\limits_{\X\times\X} K\dx\nu \dx\nu.
\end{align}
Due to our requirements on $K$, the discrepancy $\mathscr{D}_K$ is a metric, see \cite{SGFSL10}.
In particular, it holds $\mathscr{D}_K(\mu,\nu)=0$ if and only if $\mu = \nu$.
In the following remark, we require the space $C_0(\X)$ of continuous functions decaying to zero at infinity.

%If $\mu_n \weakly \mu$ and $\nu_n \weakly \nu$ as $n\rightarrow \infty$, then also $\mu_n \otimes \nu_n \weakly \mu \otimes \nu$.
%Therefore, the continuity of $K$ implies that $\lim_{n \rightarrow \infty} \mathscr{D}_K(\mu_n,\nu_n) = \mathscr{D}_K(\mu,\nu)$, so that $\mathscr{D}_K$ is  continuous with respect to weak convergence in both arguments.
\begin{Remark}\label{rem:Disc_Conv}
Let $K(x,y) = \psi(x-y)$ with $\psi \in C_0(\X) \cap L^1(\X)$ and assume there exists an $l \in \N$ such that
\[\int_\X \frac{1}{\hat \psi(x) (1 + \vert x \vert)^l} \dx x < \infty,\]
where $\hat \psi$ denotes the Fourier transform of $\psi$.
Then, the discrepancy $\mathscr D_K$ metrizes the weak topology on $\mathcal P (\R^m)$, see \cite{SGFSL10}.
\end{Remark}

\paragraph{Optimal Transport}
For our theoretical investigations, Wasserstein distances turn out to be a more appropriate metric than discrepancies.
This is mainly due to the fact that spatial distance is directly encoded into this metric.
Actually, both concepts can be related to each other under relatively mild conditions, see also Remark~\ref{rem:Disc_Conv}.
A more detailed overview on the topic can be found in \cite{Ambrosio,S2015}.
Let $\mu, \nu \in \mathcal P(\X)$ and $c \in C(\X \times \X)$ be a non-negative and continuous function. 
Then, the \emph{Kantorovich problem of optimal transport} (OT) reads
 \begin{equation}\label{Monge_Kantorovich_problem}
 \OT(\mu,\nu) \coloneqq \inf_{\pi \in\Pi(\mu,\nu)} \int_{\X \times \X} c(x,y) \dx \pi(x,y),
\end{equation}
where $\Pi(\mu,\nu)$ denotes the weakly compact set of joint probability measures $\pi$ on $\X\times\X$ with marginals $\mu$ and $\nu$.
Recall that the OT functional \smash{$\pi \mapsto  \int_{\X \times \X} c\, \dx \pi$} is weakly lower semi-continuous, 
\eqref{Monge_Kantorovich_problem} has a solution and every such minimizer $\hat \pi$ is called optimal transport plan.
Note that optimal transport plans are in general not unique.

For $\mu,\nu \in \mathcal P_p(\X) \coloneqq \{\mu \in \mathcal P(\X): \int_\X \vert x \vert^p \dx \mu(x)<\infty\}$, $p \in [1,\infty)$, the \emph{$p$-Wasserstein distance} $W_p$ is defined by
\begin{align} \label{eq:OTprimal}
W_p(\mu,\nu) \coloneqq \biggl( \min_{\pi \in \Pi(\mu,\nu)} 
\int_{\X \times \X} \vert x-y \vert^p \mathrm{d} \pi(x,y) \biggr)^\frac{1}{p}.
\end{align}
Recall that this defines a metric on $\mathcal P_p(\X)$, which satisfies $W_p(\mu_m,\mu) \to 0$ if and only if $\int_\X \vert x \vert ^p \dx \mu_m(x) \to \int_\X \vert x \vert^p \dx \mu(x)$ and $\mu_m\weakly \mu$.
If $\mu, \nu \in \mathcal P_p(\X)$ and at least one of them has a density with respect to the Lebesgue measure, then \eqref{eq:OTprimal} has a unique solution for any $p \in (1, \infty)$.
For $\mu,\nu \in \mathcal P(\R^m)$ and sets $A, B \subset \R^m$, we estimate
\begin{align}
    W_p^p(\mu,\nu) &= \min_{\pi \in\Pi(\mu,\nu)} \int_{\R^n \times \R^n} \vert x-y \vert^p \dx \pi(x,y) \geq \min_{\pi \in\Pi(\mu,\nu)} \int_{A \times B} \vert x-y \vert^p \dx \pi(x,y)\label{eq:EstWasserstein}\\ 
    & \geq \dist(A,B)^p \min_{\pi \in\Pi(\mu,\nu)} \pi(A \times B) \geq \dist(A,B)^p \max\bigl\{\mu(A) - \nu(B^c), \nu(B) - \mu(A^c)\bigr\}\notag.
\end{align}

\begin{Remark}\label{rem:DiscOT}
As mentioned in Remark~\ref{rem:Disc_Conv}, $\mathscr D_K$ metrizes the weak convergence under certain conditions.
Further, convergence of the $p$-th moments always holds if all measures are supported on a compact set.
In this case, weak convergence, convergence in the Wasserstein metric and convergence in discrepancy are all equivalent.
As these requirements are usually fulfilled in practice, most theoretical properties of INNs based on $W_1$ carry over to $\mathscr D_K$. 
\end{Remark}

\paragraph{Invertible Neural Networks}
Throughout this paper, an INN $T\colon \R^d\to\R^n\times\R^k$, $k=d-n$, is constructed as composition $T = T_L \circ P_L \circ \ldots \circ T_1 \circ P_1 $, where the $(P_l)_{l=1}^L$ are random (but fixed) permutation matrices and the $T_l$ are defined by
\begin{equation}\label{eq:DefBlock}
    T_l(x_1,x_2) = (v_1,v_2) = \bigl(x_1 \odot \exp(s_{l,2}(x_2))+t_{l,2}(x_2),\, x_2 \odot \exp(s_{l,1}(v_1))+t_{l,1}(v_1)\bigr),
\end{equation}
where $x$ is split arbitrarily into $x_1$ and $x_2$ (for simplicity of notation we assume that $x_1 \in \R^n$ and $x_2 \in \R^k$) and $\odot$ denotes the Hadamard product, see \cite{ardizzone2018analyzing}.
Note that there is a computational precedence in this construction, i.e., the computation of $v_2$ requires the knowledge of $v_1$.
The continuous functions $s_{l,1}\colon \mathbb{R}^n \rightarrow \mathbb{R}^k$, $s_{l,2}\colon \mathbb{R}^k \rightarrow \mathbb{R}^n$, $t_{l,1}\colon\mathbb{R}^n \rightarrow \mathbb{R}^k $ and $t_{l,2}\colon \mathbb{R}^k \rightarrow \mathbb{R}^n$ are NNs and do not need to be invertible themselves.
In order to emphasize the dependence of the INN on the parameters $u$ of these NNs, we use the notation $T(\cdot\,;u)$.
The inverse of $T_l$ is explicitly given by
\begin{equation}\label{eq:DefInvBlock}
    T_l^{-1}(v_1,v_2) = (x_1,x_2) = \bigl((v_1 - t_{l,2}(x_2)) \oslash \exp(s_{l,2}(x_2)),\, (v_2-t_{l,1}(v_1)) \oslash \exp(s_{l,1}(v_1))\bigr),
\end{equation}
where $\oslash$ denotes the point-wise quotient.
Hence, the whole INN is invertible.
Clearly, $T$ and $T^{-1}$ are both continuous maps. 

%However, our analysis is not limited to this particular type of INNs.
%Essentially, we only require parameterized maps that approximate bijective mappings between random vectors $X$ and $(Y,Z)$.
Other architectures for constructing INNs include the real NVP architecture introduced in \cite{DBLP:conf/iclr/DinhSB17}, where the layers have the simplified form
\[T_l(x_1,x_2) = \bigl(x_1,\, x_2 \odot \exp(s(x_1))+t(x_1)\bigr),\]
or an approach that builds on making residual neural networks invertible, see \cite{inv_resnet}.
Compared to real NVP layers, the chosen approach allows more representational freedom in the first component.
Unfortunately, this comes at the cost of harder Jacobian evaluations, which are not needed in this paper though.
Invertible residual neural networks have the advantage that they come with sharp Lipschitz bounds.
In particular, they rely on spectral normalisation in each layer, effectively bounding the Lipschitz constant. 
Further, there are neural ODEs \cite{chen2018neural}, where a black box ODE solver is used in order to simulate the output of a continuous depth residual NN, and the GLOW \cite{glow} architecture, where the $P_l$ are no longer fixed as for real NVP and hence also learnable parameters.
%Note that most of these architectures have been applied to generative modelling.
%By training such INNs with our loss functions introduced in the next section, they can in principle be applied for inverse problems, too. 

\section{Analyzing the Continuous Optimization Problem}\label{sec:INN}
Given a random vector $(X,Y)\colon \Omega\to\R^d\times\R^n$ with realizations $(x_i,y_i)_{i=1}^N$, our aim is to recover the regular conditional distribution $P_{(X|Y)}(y,\cdot)$ for arbitrary $y$ based on the samples.
For this purpose, we want to find a homeomorphism $T\colon\R^d\to\R^n\times \R^k$ such that for some fixed $Z\colon \Omega\to\R^k$, $k=d-n$, and any realization $y$ of $Y$, the expression $T^{-1}(y,Z)$ is approximately $P_{(X|Y=y)}$-distributed, see \cite{ardizzone2018analyzing}.
Note that the authors of \cite{ardizzone2018analyzing} assume $Z \sim \mathcal N(0,I_k)$, but in fact the approach does not rely on this particular choice.
Here, we model $T$ as homeomorphism with parameters $u$, for example an INN, and choose the best candidate by optimizing over the parameters.
To quantify the reconstruction quality in dependence on $u$, three loss functions are constructed according to the following two principles:
\begin{enumerate}[(i)]
\item First, the data should be matched, i.e., $T_y(X;u)\approx Y$, where $T_y$ denotes the output part corresponding to $Y$.
For inverse problems, this is equivalent to learning the forward map.
Here, the fit between $T_y(X;u)$ and $Y$ is quantified by the squared $L^2$-norm
$$
L_y(u)=E\bigl(\vert T_y(X;u)-Y\vert^2\bigr).
$$
\item Second, the distributions of $X$ and $(Y,Z)$ should be close to the respective distributions $T^{-1}(Y,Z;u)$ and $T(X;u)$.
To quantify this, we use a metric $D$ on the space of probability measures and end up with the loss functions
\begin{equation}\label{eq:LxandLz}
L_x(u)=D\bigl(P_{T^{-1}(Y,Z;u)},P_{X}\bigr) \quad\text{and}\quad L_z(u)=D\bigl(P_{T(X;u)},P_{(Y,Z)}\bigr).
\end{equation}
Actual choices for $D$ are discussed in Section~\ref{sec:NumImp}.
\end{enumerate}
In practice, we usually minimize a conic combination of these objective functions, i.e., for some $\lambda=(\lambda_1,\lambda_2)\in\R_{\geq0}^2$ we solve
$$
\hat u\in \argmin_{u} \bigl\{L_y(u)+\lambda_1 L_x(u)+\lambda_2 L_z(u)\bigr\}.
$$

First, note that $L_y(u) = 0 \Leftrightarrow T_y(X;u) = Y$ a.e.~and
\[L_x(u) = 0 \Leftrightarrow T^{-1}(Y,Z;u) \overset{d}{=} X \Leftrightarrow (Y,Z) \overset{d}{=} T(X;u) \Leftrightarrow L_z(u)=0.\]
Further, the following remark indicates that controlling the Lipschitz constant of $T^{-1}$ also enables us to estimate the $L_x$ loss in terms of $L_z$.
\begin{Remark}\label{Rem:Lx_Est}
Note that it holds $P_{T^{-1}(Y,Z;u)} = {T^{-1}}_\# P_{(Y,Z)}$ and $P_{T(X;u)} = T_\# P_X$.
Let $\hat \pi$ denote an optimal transport plan for $W_1(P_{(Y,Z)},P_{T(X;u)})$.
If $T^{-1}$ is Lipschitz continuous, we can estimate
\begin{align}
W_1\bigl(P_X,P_{T^{-1}(Y,Z;u)}\bigr) &= W_1({T^{-1}}_\# P_{T(X;u)},{T^{-1}}_\# P_{(Y,Z)})\\
&\leq \int_{\R^d} \vert x - y \vert \dx (T^{-1} \times T^{-1})_\# \hat \pi(x,y)\\
&= \int_{\R^{n+k}} \vert T^{-1}(x) - T^{-1}(y) \vert \dx \hat \pi(x,y) \leq \Lip(T^{-1}) W_1\bigl(P_{(Y,Z)},P_{T(X;u)}\bigr),\end{align}
i.e., we get $L_x(u) \leq \Lip(T^{-1}) L_z(u)$ when choosing $D = W_1$.
\end{Remark}
Under the assumption that all three loss functions are zero, the authors of \cite{ardizzone2018analyzing} proved that $T^{-1}(y,Z) \sim P_{(X|Y)}(y,\cdot)$.
In other words, $T$ can be used to efficiently sample from the true posterior distribution $P_{(X|Y)}(y,\cdot)$.
However, their proof requires that the distributions of the random vectors $X$ and $Z$ have densities and that $Z$ is independent of $Y$.
As we are going to see in Section~\ref{sec:NumImp}, it is actually beneficial to choose $Z$ dependent on $Y$.
We overcome the mentioned limitations in the next theorem and provide a rigorous prove of this important result.
\begin{theorem}\label{sampl_standard_normal}
Let $X\colon\Omega\to\R^d$, $Y\colon\Omega\to \R^n$ and $Z\colon\Omega\to\R^k$ be random vectors. 
Further, let $T=(T_y,T_z)\colon \R^d\to\R^n\times\R^k$ be a homeomorphism such that \smash{$T\circ X\overset{d}{=}(Y,Z)$} and $T_y \circ X = Y$ a.e.
%Then, we have for $P_Y$-every $y \in \Y$ and any random vector $Z^\prime \sim P_{(Z|Y=y)}$ that $T^{-1}(y,Z^\prime)\sim P_{(X|Y=y)}$.
Then, it holds
\[P_{(X|Y)}(y,B) = {T^{-1}}_\#\bigl(\delta_y \otimes P_{(Z|Y)}(y,\cdot)\bigr)(B)\]
for all $B \in \mathcal B (\R^{d})$ and $P_Y$-a.e.~$y \in \R^n$. 
\end{theorem}
\begin{proof}
%We may restrict our space $\Omega_2 = \{ \omega \in \Omega: T_y(X(\omega)) = Y(\omega)\}$.
% By assumption it holds that $P(\Omega_2) = 1$ as $T_y(X) = Y$ holds true $P_{(X,Y)}$-a.e.
As $T$ is a homeomorphism, it suffices to show \smash{$P_{(X|Y)}(y,T^{-1}(B)) = (\delta_y \otimes P_{(Z|Y)}(y,\cdot))(B)$} for all $B \in \mathcal B (\R^{n+k})$ and $P_Y$-a.e.~$y \in \R^n$. 
Using \eqref{eq:RegPush} and integrating over $y$, we can equivalently show that for all $A \in \mathcal B (\R^{n})$ and all $B \in \mathcal B (\R^{n+k})$ it holds
\[\int_A P_{(T \circ X|Y)}(y,B)\dx P_Y(y) = \int_A \delta_y \otimes P_{(Z|Y)}(y,\cdot)(B) \dx P_Y(y).\]
As both $\int_A P_{(T \circ X|Y)}(y,\cdot)\dx P_Y(y)$ and $ \int_A \delta_y \otimes P_{(Z|Y)}(y,\cdot)(\cdot) \dx P_Y(y)$ are finite measures, we only need to consider sets of the form $B = B_1 \times B_2$ with $B_1 \in \mathcal B(\R^n)$ and $B_2 \in \mathcal B(\R^k)$, see \cite[Lemma 1.42]{K2008}.

By \eqref{def:RegCondDist} and since $T_y \circ X = Y$ a.e., we conclude
\begin{align*}
\int_A P_{(T \circ X|Y)}(y,B)\dx P_Y(y)  = &\int_A E\bigl(1_{B}(T\circ X)| Y=y\bigr) \dx P_Y(y)\\ 
= &\int_A E\bigl(1_{B_1}(Y)1_{B_2}(T_z \circ X) | Y=y\bigr) \dx P_Y(y).
\end{align*}
For all $C \in \sigma(Y) = Y^{-1}(\mathcal B(\R^n))$ it holds that $\omega \in C$ if and only if $Y(\omega) \in Y(C)$.
As \smash{$T \circ X\overset{d}{=}(Y,Z)$} and $T_y \circ X=Y$ a.e., we obtain
\begin{align}
    &\int_C E(1_{B_1}(Y)1_{B_2}(T_z \circ X) | Y) \dx P =\int_\Omega 1_C 1_{B_1}(Y)1_{B_2}(T_z\circ X) \dx P\\
    =&\int_\Omega 1_{B_1 \cap Y(C)}(Y)1_{B_2}(T_z \circ X) \dx P = \int_\Omega 1_{B_1 \cap Y(C)}(T_y \circ X)1_{B_2}(T_z \circ X) \dx P\\ 
    =&\vphantom{\int_\Omega} P_{T \circ X}\bigl((B_1\cap Y(C))\times B_2\bigr) =\vphantom{\int_\Omega} P_{(Y,Z)}\bigl((B_1\cap Y(C))\times B_2\bigr)\\
    =&\int_\Omega 1_{B_1 \cap Y(C)}(Y)1_{B_2}(Z) \dx P = \int_C E(1_{B_1}(Y)1_{B_2}(Z) | Y) \dx P.
\end{align}
In particular, this implies that $E(1_{B_1}(Y)1_{B_2}(T_z \circ X) | Y) = E(1_{B_1}(Y)1_{B_2}(Z) | Y)$ and hence also  $E(1_{B_1}(Y)1_{B_2}(T_z \circ X) | Y = y) = E(1_{B_1}(Y)1_{B_2}(Z) | Y=y)$ for $P_Y$-a.e.~$y$.
By \eqref{def:RegCondDist} and \eqref{eq:PullOut}, we get
\begin{align*}
    \int_A P_{(T \circ X|Y)}(y,B)\dx P_Y(y) =& \int_A E(1_{B_1}(Y)1_{B_2}(Z) | Y = y) \dx P_Y(y)\\
    =&\int_A 1_{B_1}(y)E(1_{B_2}(Z) | Y = y) \dx P_Y(y) = \int_{A\cap B_1}P_{(Z|Y)}(y,B_2) \dx P_Y(y)\\ =& \int_A \delta_{y}(B_1)P_{(Z|Y)}(y,B_2) \dx P_Y(y),
\end{align*}
which is precisely the claim.
\end{proof}

Under some additional assumptions, we are able to show the following error estimate based on the TV norm $\Vert \mu \Vert_{\text{TV}} \coloneqq \sup_{B \in \mathcal B(R^n)} \vert \mu(B) \vert$.
Note that the restrictions and conditions are still rather strong and do not entirely fit to our problem.
However, to the best of our knowledge, this is the first attempt in this direction at all.
\begin{theorem}\label{thm:TV_err}
Let $X\colon\Omega\to\R^d$, $Y\colon\Omega\to \R^n$ and $Z\colon\Omega\to\R^k$ be random vectors such that $\vert (Y,Z) \vert \leq r$ and $\vert (Y,X) \vert \leq r$ almost everywhere.
Further, let $T$ be a homeomorphism with $E(\vert T_y(X;u)-Y\vert^2) \leq \epsilon$ and $\Vert P_{(Y,Z)} - P_{(Y,T_z(X))} \Vert_{\text{TV}} \leq \delta$.
Then it holds for all $A$ with $P_Y(A) \neq 0$ that 
$$W_1\bigl(P_{(X| Y\in A)}, {T^{-1}}_{\#}P_{(Y,Z | Y \in A)}\bigr) \leq \Lip(T^{-1}) \Bigl( \frac{\epsilon^{1/2}}{P_Y(A)^{1/2}} + \frac{2r \Lip(T)}{P_Y(A)}\delta\Bigr).$$
\end{theorem}
\begin{proof}
Due to the bijectivity of $T$ and the triangle inequality, we may estimate
\begin{align*}
&W_1\left(P_{(X| Y\in A)}, {T^{-1}}_{\#}P_{(Y,Z | Y \in A)}\right) 
\leq \Lip(T^{-1}) 
W_1\left(P_{(T(X) | Y \in A)}, P_{(Y,Z | Y \in A)}\right) 
\\ \leq&\Lip(T^{-1}) \Bigl( W_1\left(P_{(T(X) | Y \in A)},P_{(Y,T_z(X) | Y \in A)}\right) + W_1\left(P_{(Y,T_z(X) | Y \in A)}, P_{(Y,Z | Y \in A)}\right)\Bigr).
\end{align*}
Note that $\{Y \in A\} \subset \Omega$ can be equipped with the measure $Q \coloneqq P/P_Y(A)$.
Then, the random variables $T(X)$ and $(Y,Z)$ restricted to $\{Y \in A\}$ have distributions $Q_{T(X)} = P_{(T(X)|Y \in A)}$ and $Q_{(Y,T_z(X))} =P_{((Y,T_z(X))|Y \in A)}$, respectively.
Hence, we may estimate by Hölders inequality
\begin{align}
W_1\left(P_{(T(X) |Y \in A)},P_{((Y,T_z(X)) | Y \in A)}\right) &\leq \mathbb{E}_{Q}(\vert T_y(X) - Y\vert)\leq \bigl(\mathbb{E}_{Q}(1) \mathbb{E}_{Q}(\vert T_y(X) - Y\vert^2)\bigr)^{1/2}\\& \leq \frac{1}{P_Y(A)^{1/2}} \eps^{1/2}.
\end{align}

As $\vert (Y,Z) \vert \leq r$ and $\vert (Y,T_z(X) \vert \leq r\Lip(T)$ almost everywhere, we can use \cite[Theorem 4]{prob_measures_overview} to control the Wasserstein distance by the total variation distance and estimate
\begin{align*}
&W_1(P_{((Y,T_z(X))| Y \in A)}, P_{((Y,Z) | Y \in A)}) \leq 2r\Lip(T) \bigl\Vert P_{((Y,T_z(X))| Y \in A)} - P_{((Y, Z) | Y \in A)} \bigr\Vert_{\text{TV}} \\
=& \frac{2r\Lip(T)}{P_Y(A)} \sup_{B} \bigl|P_{(Y,T_z(X))}\bigl(B \cap (A \times \R^k)\bigr) - P_{(Y,Z)}\bigl(B \cap (A \times \R^k)\bigr)\bigr| \\
\leq& \frac{2r\Lip(T)}{P_Y(A)} \Vert P_{(Y,T_z(X))} - P_{(Y,Z)}\Vert_{\text{TV}} .
\end{align*}
Combining the estimates implies the result.
\end{proof}

\paragraph{Instability with Standard Normal Distributions}
For numerical purposes, it is important that both $T$ and $T^{-1}$ have reasonable Lipschitz constants.
Here, we specify a case where using $Z\sim \mathcal N(0,I_k)$ as latent distribution leads to exploding Lipschitz constants of $T^{-1}$.
Note that our results are in a similar spirit as \cite{jaini2020tails} and \cite{pmlr-v119-cornish20a}, where also the Lipschitz constants of INNs are investigated.
Due to its useful theoretical properties, we choose $W_1$ as metric $D$ throughout this paragraph.
Clearly, the established results have similar implications for other metrics that induce the same topology.
The following general proposition is a first step towards our desired result.

\begin{Proposition}\label{prop:dist_est}
Let $\nu = \mathcal N(0,I_n) \in \mathcal P(\R^n)$ and $\mu \in \mathcal P(\R^m)$ be arbitrary.
Then, there exists a constant $C$ such that for any measurable function $T\colon \R^m \to \R^n$ satisfying $W_1(T_\# \mu,\nu) \leq \epsilon$ and disjoint sets $Q,R$ with $\min\{\mu(Q),\mu(R)\} \geq\tfrac{1}{2} - \tfrac{\delta}{2}$ it holds for $\delta, \epsilon \geq 0$ small enough that
\[\dist\bigl(T(Q),T(R)\bigr) \leq C \bigl(\delta + \epsilon^\frac{n}{n+1}\bigr)^\frac{1}{n}.\]
\end{Proposition}
\begin{proof}
If $T(Q) \cap T(R) \neq \emptyset$, the claim follows immediately.
In the following, we denote the volume of $B_1(0) \subset \R^n$ by $V_1$ and the normalizing constant for $\nu$ with $C_\nu$.
Provided that $\delta, \epsilon <1/4$, we conclude for the ball $B_r(0)$ with $\nu(B_r(0))>3/4$ by \eqref{eq:EstWasserstein} that $1/4 > \dist(T(Q), B_r(0))/8$ and similarly for $T(R)$.
Hence, there exists a constant $s$ independent of $\delta$, $\epsilon$ and $T$ such that $\dist(T(Q),0) \leq s$ and $\dist(T(R),0) \leq s$.
Possibly increasing $s$ such that $s>\max(1, \ln(C_\nu V_1))$, we choose the constant $C$ as
\[C=4\biggl(\frac{\exp(2s^2)}{C_\nu V_1}\biggr)^\frac{1}{n} > 4.\]
Assume for fixed $\delta$, $\epsilon$ that
\[r^{n} \coloneqq \frac{2^{n}\exp(2s^2)}{C_\nu V_1}\bigl(\delta + \epsilon^\frac{n}{n+1} \bigr) < \frac{\dist\bigl(T(Q),T(R)\bigr)^{n}}{2^{n}}.\]
If $\epsilon$ and $\delta$ are small enough, it holds $r \leq s$.
As $\dist(T(Q) \cap B_s(0), T(R) \cap B_s(0)) >2r$, there exists $x \in \R^n$ such that the open ball $B_{r}(x)$ satisfies $B_{r}(x) \cap T(Q) = B_{r}(x) \cap T(R) = \emptyset$ and $B_{r}(x) \subset B_{2s}(0)$.
Further, we estimate
\[\nu\bigl(B_{r/2}(x)\bigr)= C_\nu \int_{B_{r/2}(x)} \exp(-\vert x \vert ^2/2) \dx x \geq C_\nu\exp(-2s^2)V_1 \frac{r^n}{2^n} \geq \delta + \epsilon^\frac{n}{n+1}\]
as well as $\dist(T(Q),B_{r/2}(x)) \geq r/2$ and $\dist(T(R),B_{r/2}(x)) \geq r/2$.
Using \eqref{eq:EstWasserstein}, we obtain $W_1(T_{\#}\mu,\nu) \geq \frac{r}{2} \eps^{n/(n+1)}$, leading to the contradiction
\[\epsilon \geq W_1(T_\# \mu,\nu)\geq \frac{r}{2} \epsilon^\frac{n}{n+1} \geq \Bigl( \frac{\exp(2s^2)}{C_\nu V_1}\Bigr)^\frac{1}{n} \epsilon > \epsilon.\]
\end{proof}
Note that the previous result is not formulated in its most general form, i.e., the mass on $Q$ and $R$ can be different.
Further, the measure $\nu$ can be exchanged as long as  the mass is concentrated around zero.
Based on Proposition~\ref{prop:dist_est}, we directly obtain the following corollary.

\begin{corollary}\label{cor_lipschitz}
Let $\mu \in \mathcal P(\R^m)$ be arbitrary.
Then, there exists a constant $C$ such that for any measurable, invertible function $T\colon \R^m \to \R^n$ satisfying $W_1(T_\# \mu, \mathcal N(0,I_n)) \leq \epsilon$ and disjoint sets $Q,R$ with $\min\{\mu(Q),\mu(R)\} \geq\tfrac{1}{2} - \tfrac{\delta}{2}$ it holds for $\delta, \epsilon \geq 0$ small enough that
\[\Lip({T^{-1}}) \geq \frac{\dist(Q,R)}{2C}\bigl(\delta + \epsilon^\frac{n}{n+1}\bigr)^{-\frac{1}{n}}.\]
\end{corollary}
\begin{proof}
Due to Proposition~\ref{prop:dist_est}, it holds
\[\Lip({T^{-1}}) \geq \sup_{x_1 \in T(Q), x_2 \in T(R)} \frac{\vert T^{-1}(x_1) - T^{-1}(x_2)\vert}{\vert x_1-x_2\vert} \geq \frac{\dist(Q,R)}{2C\bigl(\delta + \epsilon^\frac{n}{n+1}\bigr)^{\frac{1}{n}}}.\]
\end{proof}

Now, we are finally able to state the desired result for INNs with $Z \sim \mathcal N(0,I_k)$.
In particular, the following result also holds if we condition only on a single measurement $Y=y$ occurring with positive probability.
\begin{corollary}\label{cor_lip_inv}
For $A \subset \R^n$ with $P_Y(A) \neq 0$, $\diam(A) = \rho$ set $\mu \coloneqq \int_A P_{X|Y}(y,\cdot) \dx P_Y(y)/P_Y(A)$.
Further, choose disjoint sets $Q,R$ with $\min\{\mu(Q),\mu(R)\} \geq\tfrac{1}{2} - \tfrac{\delta}{2}$.
If $T\colon \R^d \rightarrow \R^{n+k}$ is a homeomorphism such that $W_1({T_z}_\# \mu,\mathcal N(0,I_k)) \leq \epsilon$ and $T_y(Q \cup R) \subset A$, we get for $\delta, \epsilon \geq 0$ small enough that the Lipschitz constant of $T^{-1}$ satisfies 
$$\Lip({T^{-1}}) \geq \frac{\dist(Q,R)}{2C\bigl(\delta + \epsilon^\frac{k}{k+1}\bigr)^{\frac{1}{k}}+\rho}.$$
\end{corollary}
\begin{proof}
Similar as before, $\Lip(T^{-1})$ can be estimated by
\begin{equation}\label{eq:Lip}
    \Lip({T^{-1}}) \geq \sup_{x_1 \in Q, x_2 \in R} \frac{\dist(Q,R)}{\vert T(x_1) - T(x_2) \vert}.
\end{equation}
Due to Proposition~\ref{prop:dist_est} and since $T_y(Q \cup R) \subset A$, we get
\[\inf_{x_1 \in Q, x_2 \in R} \vert T(x_1) - T(x_2) \vert  \leq  \inf_{x_1 \in Q, x_2 \in R} \vert T_z(x_1) - T_z(x_2) \vert + \vert T_y(x_1) - T_y(x_2) \vert \leq 2C \bigl(\delta + \epsilon^\frac{k}{k+1}\bigr)^{1/k} + \rho.\]
Inserting this into \eqref{eq:Lip} implies the claim.
\end{proof}
Loosely speaking, the condition $T_y(Q \cup R) \subset A$ is fulfilled if the predicted labels are close to the actual labels.
Furthermore, the condition that $W_1({T_z}_\# \mu, \mathcal{N}(0,I_k))$ is small is enforced via the $L_y$ and $L_z$ loss.
Consequently, we can expect that $\Lip(T^{-1})$ explodes during training if the true data distribution given measurements in $A$ is multimodal.

\paragraph{A Remedy using Gaussian Mixture Models}
The Lipschitz constant of $T^{-1}$ with multimodal input data distribution explode due to the fact that a splitting of the standard normal distribution is necessary.
As a remedy, we propose to replace the latent variable $Z \sim \mathcal N(0, I_k)$ with a Gaussian mixture model (GMM) depending on $y$.
Using a flexible number of modes in the latent space, we can avoid the necessity to split mass when recovering a multimodal distribution in the input data space.

Formally, we choose the random vector $Z$ (depending on y) as $Z \sim \sum_{i=1}^{r} p_i(y) \mathcal{N}(\mu_i, \sigma^2I_k)$, where $r$ is the maximal number of modes in the latent space and $p_i(y)$ are the probabilities of the different modes.
These probabilities $p_i$ are realised by a NN, which we denote with $w\colon \mathbb{R}^n \to \Delta_{r}$.
Usually, we choose $\sigma$ and $\mu_i$ such that the modes are nicely separated.
Note that $w$ gives our model the flexibility to decide how many modes it wants to use given the data $y$.
Clearly, we could also incorporate covariance matrices, but for simplicity we stick with scaled identity matrices.
Denoting the parameters of the NN $w$ with $u_2$, we now wish to solve the following optimization problem
\begin{equation}\label{eq:ComLoss}
    \hat u = (\hat u_1, \hat u_2) \in \argmin_{u = (u_1,u_2)} L_y(u)+\lambda_1 L_x(u)+\lambda_2 L_z(u).
\end{equation}

Next, we provide an example where choosing $Z$ as GMM indeed reduces $\Lip(T^{-1})$.
\begin{Example}
Let $X_{\sigma}\colon\Omega\to\R$ and $Z\colon\Omega\to\R$ be random vectors. 
As proven in Corollary \ref{cor_lipschitz}, we get for $X_{\sigma}\sim 0.5 \mathcal N(-1,\sigma^2)+0.5\mathcal N(1,\sigma^2)$ and $Z\sim\mathcal N(0,1)$ that for any $C>0$, there exists $\sigma_0^2>0$ such that for any $\sigma^2<\sigma_0^2$ and any $T\colon\R\to \{0\}\times\R$ with $W_1(P_{T^{-1} \circ Z},P_{X_{\sigma}})<\epsilon$ it holds that $\Lip(T^{-1})>C$. 

On the other hand, if we allow $Z$ to be a GMM, then there exists for any $\sigma^2>0$ a GMM $Z$ and a mapping $T\colon\R\to \{0\}\times\R$ such that $P_{T^{-1}\circ Z}=P_{X_{\sigma}}$ and $L_{T^{-1}}=1$, namely $Z=X_{\sigma}$ and $T=I$.
Clearly, we can also fix $Z \sim 0.5 \mathcal{N}(-1,0.1)+0.5\mathcal N(1,0.1)$.
In this case, we can find $T$ such that $P_{T^{-1}(Z)}=P_{X_{\sigma}}$ and $\Lip(T^{-1})$ remains bounded as $\sigma \to 0$.
\end{Example}

\section{Learning INNs with Multimodal Latent Distributions}\label{sec:NumImp}
In this section, we discuss the training procedure if $T$ is chosen as INN together with a multimodal latent distribution.
For this purpose, we have to specify and discretize our generic continuous problem \eqref{eq:ComLoss}.
\paragraph{Model Specification and Discretization}
First, we fix the metric $D$ in the loss functions \eqref{eq:LxandLz}.
Due to their computational efficiency, we use discrepancies $\mathscr D_K$ with a multiscale version of the inverse quadric kernel, i.e.,
\[K(x,y) = \sum_{i=1}^3 \frac{r_i^2}{\sqrt{\vert x-y \vert^2 +r_i^2}},\]
where $r_i \in \{0.05,0.2,0.9\}$.
The different scales ensure that both local and global features are captured.
Clearly, computationally more involved choices such as the Wasserstein metric $W_1$ or the sliced Wasserstein distance \cite{SlicedWasserstINN} could be used as well.
The relation between $W_p$ and $\mathscr D_K$ is discussed in Remark~\ref{rem:DiscOT}.

In the previous section, we analyzed the problem of modelling multimodal distributions with an INN in a very abstract setting by considering the data as a distribution.
In practice, however, we have to sample from the random vectors $X$, $Y$ and $Z$.
The obtained samples $(x_i,y_i,z_i)_{i=1}^m$ induce empirical distributions of the form $P\raisebox{1pt}{$\strut{}^m_{X}$} = \tfrac{1}{m} \sum_{i=1}^m \delta(\cdot - x_i)$, which are then inserted into the loss \eqref{eq:ComLoss}.
To this end, we replace $P_{T^{-1}(Y,Z;u)}$ by \smash{${T^{-1}}_\# P\raisebox{1pt}{$\strut{}^m_{(X,Y)}$}$} and $P_{T(X;u)}$ by \smash{$T_\# P\raisebox{1pt}{$\strut{}^m_X$}$}, resulting in the overall problem
\begin{equation}\label{eq:DiscLoss}
\min_{u = (u_1,u_2)} \biggl\{\frac{1}{m}\sum_{i=1}^m \vert T_y(x_i;u)-y_i\vert^2 + \lambda_1 \mathscr D_K\bigl({T^{-1}}_\# P^m_{(Y,Z)},P^m_{X}\bigr) +\lambda_2 \mathscr D_K\bigl(T_\# P^m_{X},P^m_{(Y,Z)}\bigr) \biggr\},
\end{equation}
where $T$ depends on $u_1$ and $Z$ on $u_2$.
Here, the discrepancies can be easily evaluated based on \eqref{mercer_1}, which is just a finite sum.
Note that there are also other discretization approaches, which may have more desirable statistical properties, e.g., the unbiased version of the discrepancy outlined in \cite{MMD}.
For solving \eqref{eq:DiscLoss}, we want to use the gradient descent based Adam optimizer \cite{DBLP:journals/corr/KingmaB14}.
This approach leads to a sequence $\{T^m\}_{m \in \N}$ of INNs that approximate the true solution of the problem.
Deriving the gradients with respect to $u_1$ is a standard task for NNs.
In order to derive gradients with respect to $u_2$, we need a sampling procedure for $Z \sim \sum_{i=1}^{r} p_i(y) \mathcal{N}(\mu_i, \sigma^2I_k)$ such that the samples are differentiable with respect to $u_2$.
Without the differentiability requirement, we would just draw a number $l$ from a random vector $\gamma$ with range $\{1,\ldots,r\}$ and $P(\gamma = i) = p_i$ and then sample from $\mathcal{N}(\mu_l, \sigma^2I_k)$.
In the next paragraph, the Gumbel softmax trick is utilized to achieve differentiability.

\paragraph{Differentiable Sampling According to the GMM}
Recall that the distribution function of Gumbel(0,1) random vectors is given by $F(x) = \exp(-\exp(-x))$, which allows for efficient sampling.
The Gumbel reparameterization trick \cite{jang2016categorical, maddison2016concrete} relies on the observation that if we add Gumbel(0,1) random vectors to the logits $\log(p_i)$, then the $\argmax$ of the resulting vector has the same distribution as $\gamma$.
More precisely, for independent Gumbel(0,1) random vectors $G_1,...,G_r$ it holds
$$P\bigl(\argmax_i (G_i + \log(p_i)) = k\bigr) = p_k.$$
Unfortunately, $\argmax$ leads to a non differentiable expression.
To circumvent this issue, we replace $\argmax$ with a numerically more robust version of $\softmax_t\colon \R^r \to \R^r$ given by $\softmax_t(p)_i = \exp((p_i -m)/t)/  \sum_{i=1}^{n} \exp((p_i - m)/t),$ where $m = \max_i p_t$.
This results in a probability vector that converges to a corner of the probability simplex as the temperature $t$ converges to zero.
Instead of sampling from a single component, we then sample from all components and use the corresponding linear combination to obtain a sample from $Z$.
This procedure for differentiable sampling from the GMM is summarized in Algorithm~\ref{Algo:GMM}.
In our implementation, we use $t\approx 0.1$ and decrease $t$ during training.
\begin{algorithm}[t]
    \textbf{Input:} NN $w$, $y \in \R^n$, means $\mu_i$, standard deviation $\sigma$ and temperature $t$ \\
    \textbf{Output:} vector $z \in \R^k$ that (approximately) follows the distribution $P_{Z|Y}(y,\cdot)$
    \begin{algorithmic}[1]
    \State $p = w(y)$
    \State generate $g = g_1,...,g_k$ independent Gumbel(0,1) samples
    \State $p_i = \log(p_i) + g_i$
    \State $m = \max_i p_i$
    \State $s_i = \softmax_t(p)_i =\exp((p_i-m)/t)/\sum_{i=1}^{n} \exp((p_i-m)/t)$
    \State $z = \sum_{i=1}^k s_i \mu_i +
    \sigma \mathcal{N}(0,I_k)$
    %\State ODER:$z = \sum_{i=1}^k s_i \mu_i + s_i \sigma_i
    %\mathcal{N}(0,1)$ 
    \end{algorithmic}
    \caption{Differentiable Sampling according to the GMM}
    \label{Algo:GMM}
\end{algorithm}

\paragraph{Enforcing the Lipschitz Constraint}
Fitting discrete and continuous distributions in a coupled process has proven to be difficult in many cases.
Gaujac et al.~\cite{gaujac2018gaussian} encounter the problem that both models learn with different speeds and hence tend to get stuck in local minima. 
To avoid such problems, we enforce additional network properties, ensuring that the latent space is used in the correct way.
In this subsection, we argue that a $L^2$-penalty on the subnetworks weights of the $\INN$ prevents the Lipschitz constants from exploding.
This possibly reduces the chance of reaching undesirable local minima.
Further, we want modalities to be preserved, i.e., both the latent and the input data distribution should have an equal number of modes.
If this is not the case, modes have to be either split or merged and hence the Lipschitz constant of the forward or backward map explode, respectively.
Consequently, if we use a GMM for $Z$, controlling the Lipschitz constants is a natural way of forcing the INN to use the multimodality of $Z$.

Fortunately, following the techniques in \cite{behrmann2020understanding}, we can simultaneously control the Lipschitz constants of $\INN$ blocks and their inverse given by
$$S(x_1,x_2) = \bigl(x_1 ,\, x_2 \odot \exp(s(x_1))+t(x_1)\bigr) \quad \text{ and } \quad S^{-1}(x_1,x_2) = \bigl(x_1 ,\, (x_2 - t(x_1)) \oslash \exp(s(x_1))\bigr)$$ 
via bounds on the Lipschitz constants of the NNs $s$ and $t$.
To prevent exploding values of the exponential, $s$ is clamped, i.e., we restrict the range of the output components $s_i$ to $[- \alpha, \alpha]$ by using $s_{\textnormal{clamp}}(x) = \frac{2\alpha}{\pi} \arctan(s(x)/\alpha)$ instead of $s$, see \cite[Eq.~7]{ardizzone2020conditional}.
Similarly, we clamp $t$.
Note that this does not affect the invertibility of the INN.
Then, we can locally estimate on any box $[a,b]^d$ that
\begin{equation}\label{eq:LipEst}
    \max \bigl\{\Lip(S), \Lip(S^{-1})\bigr\} \leq c + c_1 \Lip(s) + c_2 \Lip(t),
\end{equation}
where the constant $c$ is an upper bound of $g$ and $1/g$ on the clamped range of $s$, $c_1$
% =  \max(\max(|a|,|b|)  c(g'),\max(|a|,|b|)) c\bigl((\frac{1}{g})'\bigr)+c(t) c\bigl((\frac{1}{g})'))$
depends on $a$, $b$ and $\alpha$, and $c_2$ depends on $\alpha$.
Provided that the activation is 1-Lipschitz, $\Lip(s)$ showing up in \eqref{eq:LipEst} is bounded in terms of the respective weight matrices $A_i$ of $s$ by
$$\Lip(s) \leq \prod_{i=1}^{l} \Vert A_i \Vert_2,$$
which follows directly from the chain rule.
This requirement is fulfilled for most activations and in particular for the often applied ReLU activation.
Note that the same reasoning is applicable for $\Lip(t)$.

Now, we want to extend \eqref{eq:LipEst} to our network structure.
To this end, note that the layer $T_l$ defined in \eqref{eq:DefBlock} can be decomposed as $T_l = S_{1,l} \circ S_{2,l}$ with
\begin{align*}
S_{1,l}(v_1,v_2) &= \bigl(v_1 ,\, v_2  \odot  \exp \bigl(s_{1,l}(v_1)\bigr)+t_{1,l}(v_1)\bigr), \\S_{2,l}(x_1,x_2) &= \bigl(x_1 \odot \exp\bigl(s_{2,l}(x_2)\bigr)+t_{2,l}(x_2),\, x_2\bigr).
\end{align*}
Incorporating the estimate \eqref{eq:LipEst}, we can bound the Lipschitz constants for $T_l$ by $$\max \bigl\{\Lip(T_l), \Lip(T_l^{-1})\bigr\} \leq \prod_{i=1}^2 \bigl(c+ c_1 \Lip(s_{i,l}) + c_2 \Lip(t_{i,l})\bigr).$$
Then, $\Lip(T)$ and $\Lip(T^{-1})$ of the INN $T$ are bounded by the respective products of the bounds on the individual blocks.

To enforce the Lipschitz continuity, we add the $L^2$-penalty $L_{\reg} = \frac{\lambda_{\reg}}{2} \sum_{i=1}^{l} \Vert A_i \Vert^2$, where $\lambda_{\reg}$ is the regularization parameter and $(A_i)_{i=1}^{m}$ are the weight matrices of the subnetworks $s_{i,l}$ and $t_{i,l}$ of the INN $T$. 
This is a rather soft way to enforce Lipschitz continuity of the subnetworks.
Other ways are described in \cite{gulrajani2017improved} and include weight clipping and gradient penalties.
Even with the additional $L^2$-penalty, our model can get stuck in bad local minima, and it might be worthwhile to investigate other (stronger) ways to prevent mode merging.
Further, the derivation indicates that clamping the outputs of the subnetworks helps to control the Lipschitz constant.
Another way to prevent merging of modes is to introduce a sparsity promoting loss on the probabilities, namely $(\sum_i p_i^{1/2})^2$.

\paragraph{Padding}
In practice, the dimension $d$ is often different from the intrinsic dimension of the data modeled by $X \colon \Omega \to \R^d$.
As $Z \colon \Omega \to \R^k$ is supposed to describe the lost information during the forward process, it appears sensible to choose $k<d-n$ in such cases.
Similarly, if $Y\colon \Omega \to \R^n$ has discrete range, we usually assign each label to a unit vector $e_i$, effectively increasing the overall dimension of the problem.
In both cases, $d \neq n +k$ and we can not apply our framework directly.
To resolve this issue, we enlarge the samples $x_i$ or $(y_i,z_i)$ with random Gaussian noise of small variance.
Note that we could also use zero padding, but then we would lose some control on the invertibility of $T$ as the training data would not span the complete space.
Clearly, padding can be also used to increase the overall problem size, adding more freedom for learning certain structures.
This is particularly useful for small scale problems, where the capacity of the NNs is relatively small.

As they are irrelevant for the problem that we want to model, the padded components are not incorporated into our proposed loss terms.
Hence, we need to ensure that the model does not use the introduced randomness in the padding components of $x = (x_{\data},x_{\pad})$ and $z = (z_{\data},z_{\pad})$ for training.
To overcome this problem, we propose to use the additional loss
\[L_{\pad} = \bigl|T_{\pad}(x_i)- z_{i,\pad}\bigr|^2 + \bigl|x_i - T^{-1}\bigl(T_y(x_i),z_{i,\pad},T_{z_{\data}}(x_i)\bigr)\bigr|^2  + \bigl| T^{-1}_{\pad}(y_i,z_{i,\pad},z_i)\bigr|^2\]
Here, the first term ensures ensures that the padding part of $T$ stays close to zero, the second one helps to learn the correct backward mapping independent of $z_{\pad}$ and the third one penalizes the use of padding in the $x$-part.
Note that the first two terms were already used in \cite{ardizzone2018analyzing}.
Clearly, we can compute gradients of the proposed loss with respect to the parameters $u_2$ of $w$ in the GMM as the loss depends on $z_i$ in a differentiable manner.
 
\section{Experiments}\label{sec:Examples}
\renewcommand*{\thefootnote}{\arabic{footnote}}
In this section, we underline our theoretical findings with some numerical experiments.
For all experiments, the networks are chosen such that they are comparable in capacity and then trained using a similar amount of time or until they stagnate. 
The exact architectures and obtained parameters are provided as supplementary material.
Our PyTorch implementation builds up on the freely available FrEIA package\footnote{\url{https://github.com/VLL-HD/FrEIA}}.
\paragraph{8 Modes Problem}
For this problem, $X\colon \Omega \to \R^2$ is chosen as a GMM with 8 modes, i.e.,
\[X  \sim \frac{1}{8} \sum_{i=1}^{8} \mathcal N(\mu_i, 0.04I_2) \quad \text{with} \quad \mu_i = \bigl(\cos(\tfrac\pi 8(2i -1)), \sin(\tfrac\pi 8(2i -1))\bigr)\big/\sin(\tfrac\pi 8).\]
All modes are assigned one of the four colors red, blue, green and purple.
Then, the task is to learn the mapping from position to color and, more interestingly, the inverse map where the task is to reconstruct the conditional distribution $P_{(X|Y)}(y,\cdot)$ given a color $y$.
Two versions of the data set are visualized in Fig.~\ref{colorsets}, where the data points $x$ are colored according to their label. 
Here, the first set is based on the color labels from \cite{ardizzone2018analyzing} and the second one uses a more challenging set of labels, where modes of the same color are further away from each other.
For this problem, we are going to use a two dimensional latent distribution.
Hence, the problem dimensions are $d=2$, $n=4$ and $k=2$.
\begin{figure}[t]
    \centering
    \includegraphics[width=0.6\textwidth]{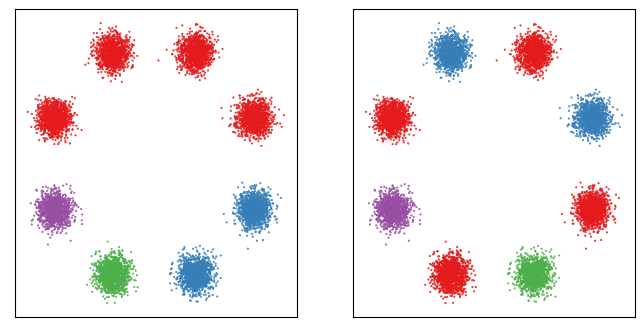}
    \caption{Positions of the 8 modes experiment with two different color label sets.}
    \label{colorsets}
\end{figure}
\begin{figure}[t]
    \centering
    \begin{subfigure}[t]{0.24\textwidth}
    \includegraphics[width=\textwidth]{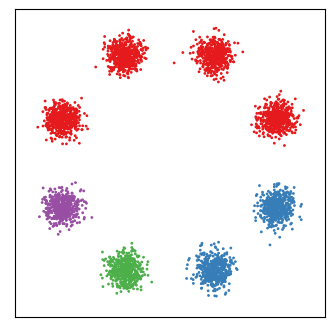}
    \caption{Ground truth.} \label{fig:gt}
    \end{subfigure}
    \begin{subfigure}[t]{0.24\textwidth}
    \includegraphics[width=\textwidth]{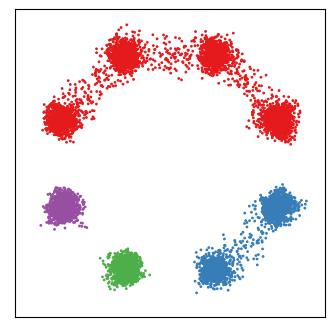}
    \caption{Reconstruction with strong $L^2$-penalty.} \label{fig:l2}
    \end{subfigure}
    \begin{subfigure}[t]{0.24\textwidth}
    \includegraphics[width=\textwidth]{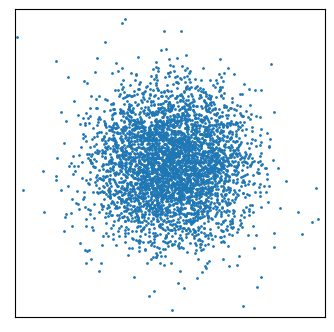}
    \caption{Latent samples with padding noise.} \label{fig:padding}
    \end{subfigure}
    \begin{subfigure}[t]{0.24\textwidth}
    \includegraphics[width=\textwidth]{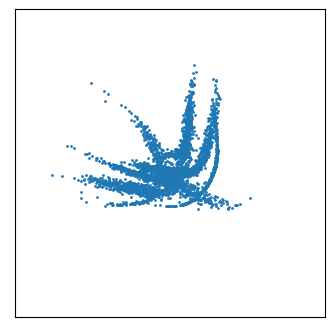}
    \caption{Latent samples without padding noise.} \label{fig:padding2}
    \end{subfigure}
    \caption{Effect of $L^2$-regularization and of padding on sampling quality.}
\end{figure}

As $d < k +n$, we have to use random padding.
Here, we increase the dimension to $16$, i.e., both input and output are padded to increase capacity.
Unfortunately, as the padding is not part of the loss terms, the argumentation in Remark~\ref{Rem:Lx_Est} is not applicable anymore, i.e., training on $L_y$ and $L_z$ is not sufficient.
This observation is supported by the fact that training without $L_x$ and $L_{\pad}$ yields undesirable results for the 8 modes problem:
The samples in Fig.~\ref{fig:padding} are created by propagating the padded data samples through the trained INN.
As enforced by $L_z$, we obtain the standard normal distribution.
However, if replace the padded components by zero, we obtain the samples in Fig.~\ref{fig:padding2}, which clearly indicates that the INN is using the random padding to fit the latent distribution.

Next, we illustrate the effect of the additional $L^2$-penalty introduced in Section~\ref{sec:NumImp}.
In Fig.~\ref{fig:l2}, we provide a sampled data distribution from an INN with $Z \sim \mathcal N(0,I_2)$ trained with large $\lambda_\reg$.
The learned model is not able to capture the multimodalities, since the $L^2$-penalty controls $\Lip(T^{-1})$.
This confirms our theoretical findings in Section~\ref{sec:INN}, i.e., that splitting up the standard normal distribution in order to match the true data distribution leads to exploding $\Lip(T^{-1})$.

\begin{figure}[t]
    \centering
    \begin{subfigure}[t]{0.49\textwidth}
        \includegraphics[width=\textwidth]{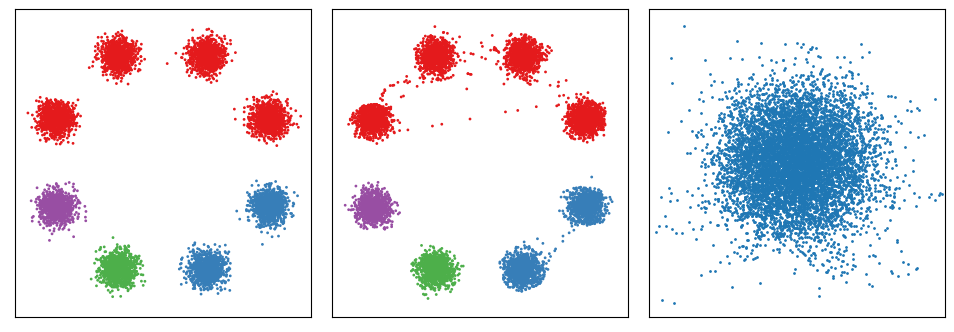}
        \caption{INN for first label set.}\label{fig:INN_comp_a}
    \end{subfigure}
    \hspace{.05cm}
    \begin{subfigure}[t]{0.49\textwidth}
        \includegraphics[width=\textwidth]{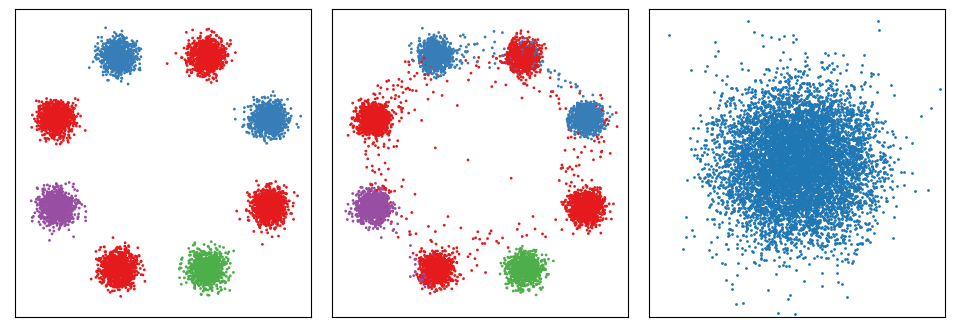}
        \caption{INN for second label set.}\label{fig:INN_comp_b}
    \end{subfigure}
    \vspace{.2cm}

    \begin{subfigure}[t]{0.49\textwidth}
        \includegraphics[width=\textwidth]{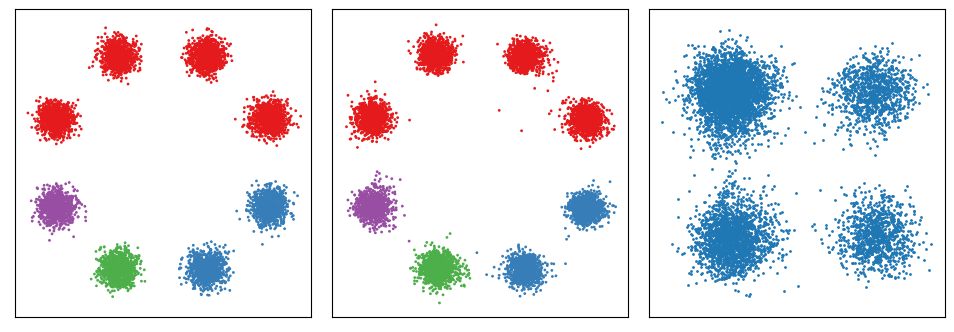}
        \caption{Multimodal INN for first label set.}\label{fig:INN_comp_c}
    \end{subfigure}
    \hspace{.05cm}
    \begin{subfigure}[t]{0.49\textwidth}
        \includegraphics[width=\textwidth]{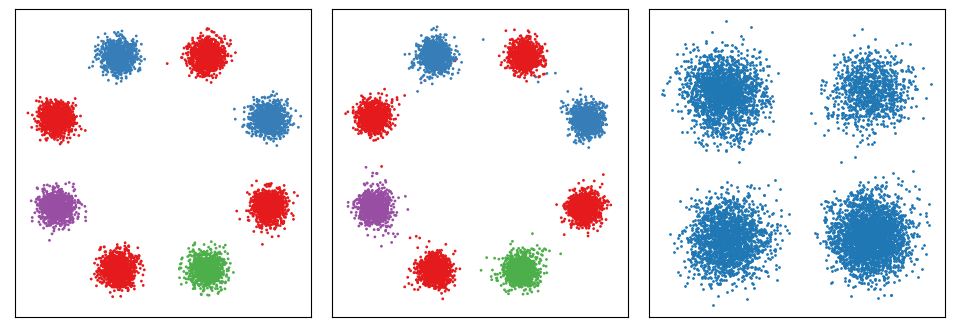}
        \caption{Multimodal INN for second label set.}\label{fig:INN_comp_d}
    \end{subfigure}

    \caption{Predicted labels, backward samples and latent samples for standard INNs and multimodal INNs.}\label{INN_comp}
\end{figure}
\begin{table}[t]
\centering
 \begin{tabular}{||c c c c||} 
 \hline
  red & blue & green & purple \\ [0.5ex] 
 \hline\hline
 0.25 & 0 & 0 & 0\\ 
 \hline
 0.25 & 0.5& 1 & 1\\
 \hline
 0.25 & 0.5 & 0 & 0 \\
 \hline
 0.25 &  0 & 0 & 0 \\
 \hline
\end{tabular}
\hspace{.5cm}
\begin{tabular}{||c c c c||} 
 \hline
  red & blue & green & purple \\ [0.5ex] 
 \hline\hline
 0.25 & 0.49 & 0 & 0 \\ 
 \hline
 0.25 & 0.51 & 0 & 0 \\
 \hline
 0.25 & 0 & 0 & 0 \\
 \hline
 0.25 &  0  & 1 & 1 \\
 \hline
\end{tabular}
\hspace{.5cm}
\caption{Learned probabilities for the results in Figs.~\ref{fig:INN_comp_c} (left) and \ref{fig:INN_comp_d} (right).\label{tab:probs}}
\end{table}
Finally, we investigate how a GMM consisting of 4 Gaussians with standard deviation $\sigma=1$ and means $\{(-4,-4),(-4,4),(4,4),(4,-4)\}$ compares against $\mathcal N (0,I_2)$ as latent distribution. 
In our experiments, we investigate the predicted labels, the backward samples and the forward latent fit for the two different label sets.
Our obtained results are depicted in Figs.~\ref{fig:INN_comp_a}-\ref{fig:INN_comp_d}, respectively.
The left column in each figure contains the label prediction of the INN, which is quite accurate in all cases.
More interesting are the reconstructed input data distributions in the middle columns.
As indicated by our theory, the standard INN has difficulties separating the modes.
Especially in Fig.~\ref{fig:INN_comp_b} many red points are reconstructed between the actual modes by the standard INN.
Taking a closer look, we observe that the sampled red points still form one connected component for both examples, which can be explained by the Lipschitz continuity of the inverse INN.
The results based on our multimodal INN in Fig.~\ref{fig:INN_comp_c} and \ref{fig:INN_comp_d} are much better, i.e., the modes are separated correctly.
Additionally, the learned discrete probabilities larger than zero match the number of modes for the corresponding color.
Note that our model learned the correct probabilities almost perfectly, see Tab.~\ref{tab:probs}.
In the right column of every figure, we included the latent samples produced by propagating the data samples through the INN.
Here, we observe that the standard INN is able to perfectly fit the standard normal distribution in both cases.
The latent fits for the multimodal INN also capture the Gaussian mixture model well, although the sizes of the modes vary a bit. 

\paragraph{MNIST image inpainting}
\begin{figure}[p]
    \centering
    \begin{subfigure}[t]{0.7\textwidth}
   \includegraphics[width=\textwidth]{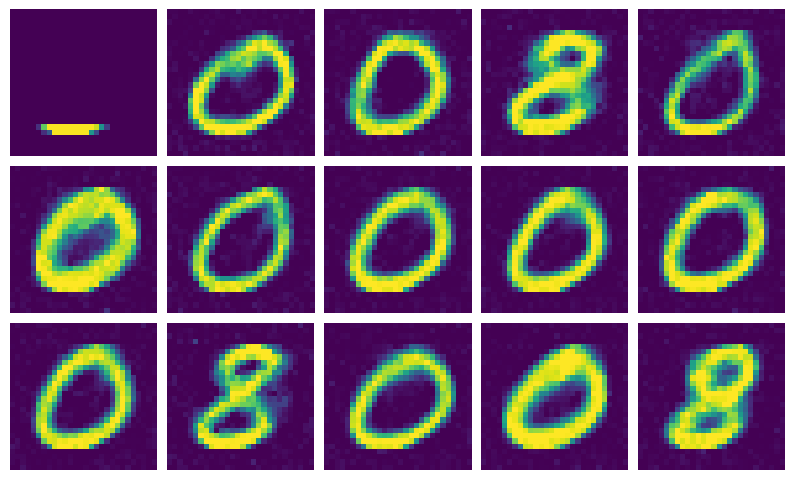}
    \label{MNIST_12}
    \end{subfigure}
    
    \begin{subfigure}[t]{0.7\textwidth}
    \includegraphics[width=\textwidth]{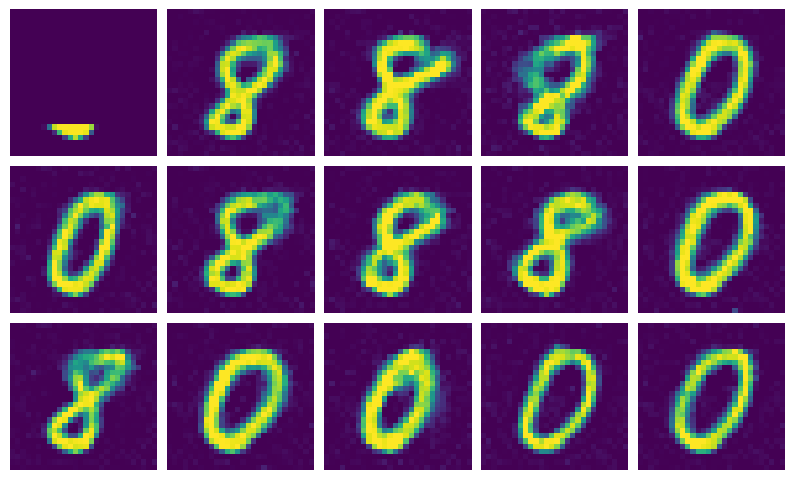}
    \label{MNIST_44}
    \end{subfigure}
    
    \begin{subfigure}[t]{0.7\textwidth}
    \includegraphics[width=\textwidth]{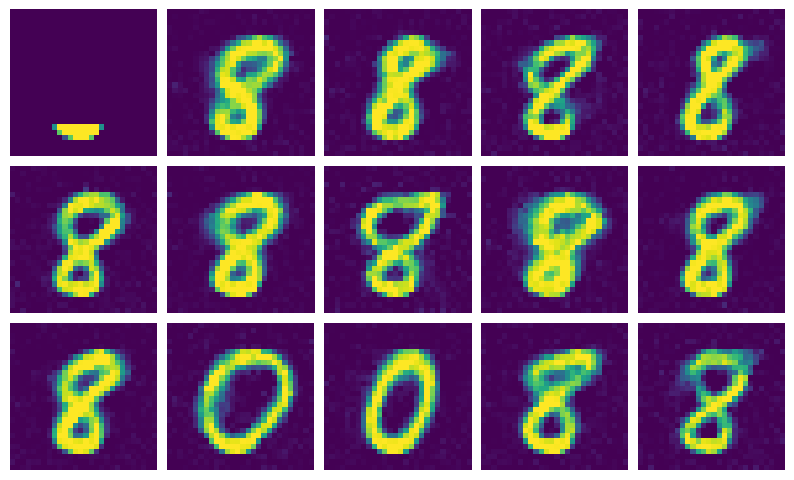}
    \label{MNIST_512}
    \end{subfigure}
    \caption{Samples created by a multimodal INN. Condition $y$ in upper left corner.}
    \label{MNIST_inpainting}
\end{figure}
To show that our method also works for higher dimensional problems, we apply it for image inpainting.
For this purpose, we restrict the MNIST dataset \cite{MNIST} to the digits 0 and 8. 
Further, the forward mapping for an image $x \in [0,1]^{28,28}$ is defined as $f\colon [0,1]^{28,28} \rightarrow [0,1]^{6,28}$ with $x \mapsto (x)_{22 \leq i \leq 27, 0 \leq j \leq 27}$, i.e., the upper part of the image is removed.
Note that the problem is chosen in such a way that it is unclear whether the lower part of the image should be completed to a zero or an eight.
However, for certain types of inputs, it is of course more likely that they belong to one or the other. 
For this example, the problem dimensions are chosen as $d=784$, $n=164$ and $k=14$, i.e., we use a $14$-dimensional latent distribution.
Further, the latent distribution is chosen as a GMM with two modes and $\sigma = 0.6$.
Unfortunately, forcing our model to learn a good separation of the modes is non-trivial and we use the following heuristic for choosing the means of the two mixture components.
We initialize the INN with random noise (of small scale) and choose the means as a scalar multiple of the forward propagations of the INN in the $Z$-space, i.e., such that they are well-separated.
This initialization turned out to be necessary for learning a clean splitting between zeros and eights.
Furthermore, it is beneficial to add noise to both the $x$-data and $y$-data to avoid artifacts in the generation process, since MNIST as a dataset itself is quite flat, see also \cite{ardizzone2020conditional}.
Although not required, we found it beneficial to replace the second term in the padding loss by $\bigl|x_i - T^{-1}\bigl(y_i,z_{i,\pad},T_{z_{\data}}(x_i)\bigr)\bigr|^2$, which serves as a backward MSE and makes the images sharper.

Our obtained results are provided in Fig.~\ref{MNIST_inpainting}.
The learned probabilities for the displayed inputs are $(0.83,0.17)$, $(0.46,0.54)$ and $(0.27,0.73)$, respectively, where the $i$-th number is the probability of sampling from the $i$-th mean in the latent space. 
Taking a closer look at the latent space, we observe that the first mode mainly corresponds to sampling a zero, whereas the second one corresponds to sampling an eight. 
A nice benefit is that the splitting between the two modes provides a probability estimate whether a sample $y$ belongs to a zero or an eight.
In order to understand if this estimate is reasonable, we perform the following sanity check.  
We pick $1000$ samples $(y_i)_{i=1}^{1000}$ from the test set and evaluate the probabilities $p(y_i)$ that those samples correspond to a zero.
For every sample $y_i$, we perform Alg.~\ref{heuristic_prob_estimate} to get an estimate $e(y_i)$ for the likeliness of being a zero.
Essentially, the algorithm performs a weighted sum over all samples with label $0$, where the weights are determined by the closeness to the sample $y_i$ based on a Gaussian kernel.
Then, we calculate the mean of $|p(y_i)-e(y_i)|$ over $1000$ random samples, resulting in an average error of around $0.10-0.11$.
Given that the mean of $|0.5-e(y_i)|$ is around $0.23$, we conclude that the estimates from our probability net are at least somewhat meaningful and better than an uneducated guess.
Intriguingly, for the images in Fig.~\ref{MNIST_inpainting} the estimates are $(0.82,0.18)$, $(0.51,0.49)$ and $(0.25,0.75)$, respectively, which nicely match the probabilities predicted by the INN.

\begin{algorithm}[t]
    \textbf{Input:} sample $y$, comparison samples $(y_i)_{i=1}^m$ with labels $l_i \in \{0,8\}$\\
    \textbf{Output:} estimate for the probability of $y$ belonging to a zero
    \begin{algorithmic}[1]
    \State $e$ = 0, $s$ = 0
    \For{$i=1,\ldots,m$}
        \If{$l_i =0$}:
        \State $e = e + \exp(-\vert y- y_i \vert^2)$
         \EndIf
    \State $s = s + \exp(-\vert y- y_i \vert^2)$
    \EndFor
    \State return $\frac{e}{s}$
    \end{algorithmic}
    \caption{Probability estimate}
    \label{heuristic_prob_estimate}
\end{algorithm}

\paragraph{Inverse Kinematics}
\begin{algorithm}[t]
    \textbf{Input:} prior $p$, some vector $y$, forward operator $f$, threshold $\eps$\\
    \textbf{Output:} $m$ samples from the approximate posterior $p(\cdot|y)$
    \begin{algorithmic}[1]
    \For{$i=1,\ldots,m$}
    \Repeat
    \State sample $x_{i}$ according to the prior $p$ and accept if $\vert y-f(x_{i})\vert^2 < \eps$
    \Until sample $x_{i}$ is accepted
    \EndFor
    \State return samples $(x_i)_{i=1}^m$
    \end{algorithmic}
    \caption{Approximate Bayesian computation}
    \label{Algo:ABC}
\end{algorithm}
In \cite{Kruse2019BenchmarkingIA} the authors propose to benchmark invertible architectures on a problem inspired by inverse kinematics.
More precisely, a vertically movable robot arm with 3 segments connected by joints is modeled mathematically as 
\begin{equation}\label{eq:Arm}
\begin{aligned}
&y_1 = h + l_1 \sin(\theta_1) + l_2 \sin(\theta_2-\theta_1) + l_3 \sin(\theta_3-\theta_2-\theta_1),\\
&y_2 = l_1 \cos(\theta_1) + l_2 \cos(\theta_2-\theta_1) + l_3 \cos(\theta_3-\theta_2-\theta_1),
\end{aligned}
\end{equation}
where the parameter $h$ represents the vertical height, $\theta_1,\theta_2, \theta_3$ are the joint angles, $l_1,l_2,l_3$ are the arm segment lengths and $y = (y_1,y_2)$ denotes the modeled end positions.
Given the end position $y$, our aim is to reconstruct the arm parameters $x = (h, \theta_1,\theta_2,\theta_3)$.
For this problem, the dimensions are chosen as $d=4$, $n=2$ and $k=2$, i.e., no padding has to be applied.
This is an interesting benchmark problem as the conditional distributions can be multimodal, the parameter space is large but still interpretable and the forward model \eqref{eq:Arm}, denoted by $f$, is more complex compared to the preceding tasks.

As a baseline, we employ approximate Bayesian computation (ABC) to get $x$ samples satisfying $y=f(x)$, see Alg.~\ref{Algo:ABC}.
For this purpose, we impose the priors $h \sim \frac{1}{2}\mathcal{N}(1,\frac{1}{64}) + \frac{1}{2}\mathcal{N}(-1,\frac{1}{64})$, $\theta_1,\theta_2,\theta_3 \sim \mathcal{N}(0,\frac{1}{4})$ and set the joint lengths to $l = (0.5,0.5,1)$.
Additionally, we sample using a INN with $Z\sim \mathcal N(0,I_2)$ and a multimodal INN with $Z \sim p_1(y) \mathcal N ((-2,0),0.09I_2) + p_2(y) \mathcal N ((2,0),0.09I_2)$.
At this point, we want to remark that computing accurate samples using ABC is very time consuming and not competitive against INNs. In particular the sampling procedure is for one fixed $y$, whereas the INN approach recovers the posteriors for all possible $y$.
%For visualization we interpolated the following four points linearly: $(0,x_1), (l_1 \cos(x_2),x_1 + l_1 \sin(x_2)),(l_1 \cos(x_2)+ l_2 \cos(x_3-x_2),x_1 + l_1 \sin(x_2)+ l_2 \sin(x_3-x_2)),( l_1 \cos(x_2) + l_2 \cos(x_3-x_2) + l_3 \cos(x_4-x_3-x_2), x_1 + l_1 \sin(x_2) + l_2 \sin(x_3-x_2) + l_3 \sin(x_4-x_3-x_2))$.https://www.overleaf.com/project/5eaad9dec803f60001e73e51

\afterpage{\begin{figure}[t]
    \begin{subfigure}[t]{0.32\textwidth}
        \includegraphics[width=\textwidth]{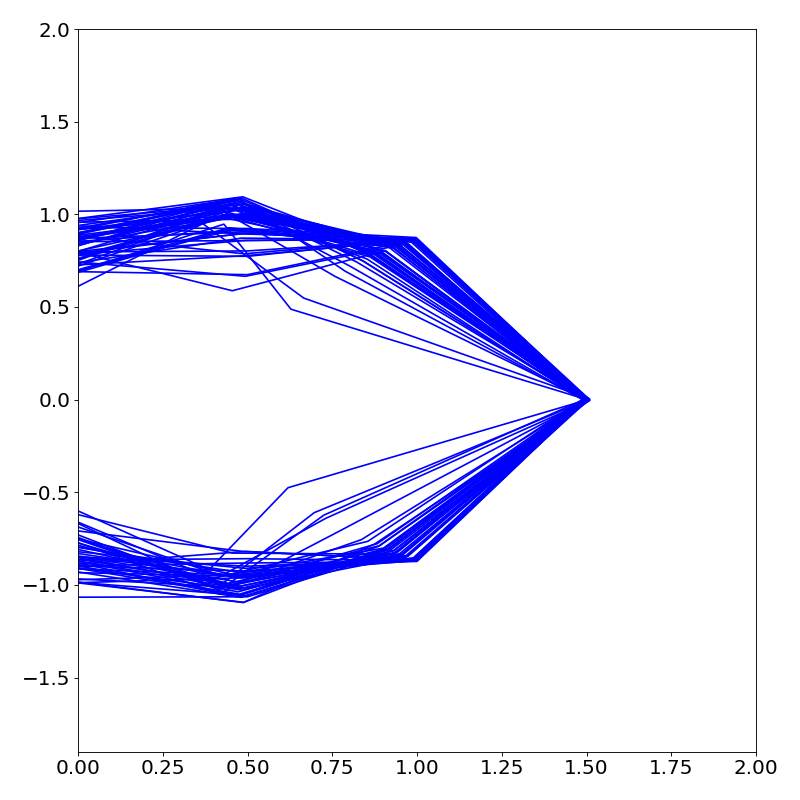}
        \caption{Metropolis--Hastings for $y$.}\label{fig:rej_1}
    \end{subfigure}
    \begin{subfigure}[t]{0.32\textwidth}
        \includegraphics[width=\textwidth]{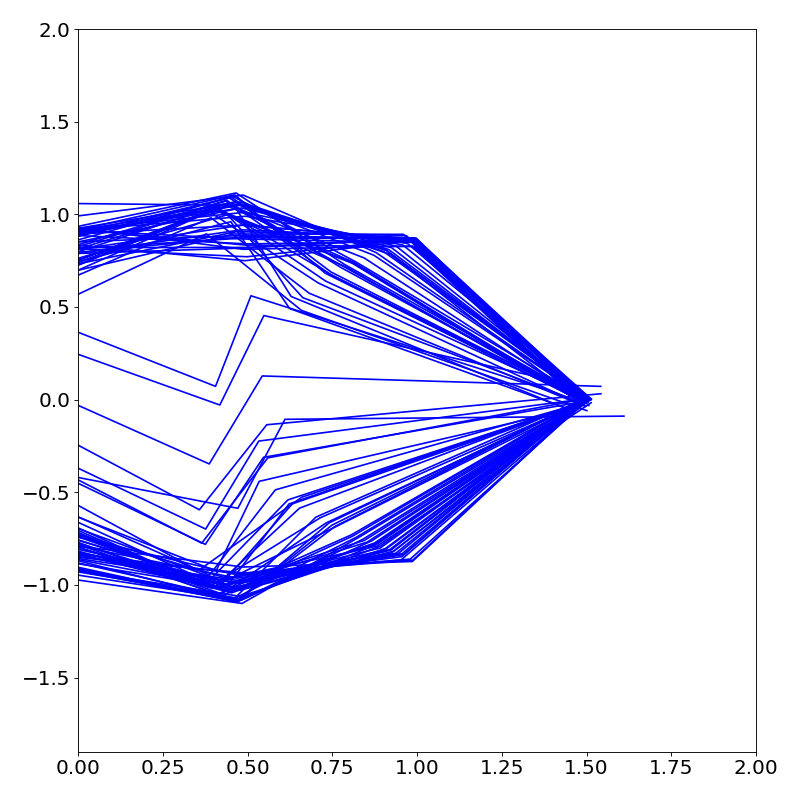}
        \caption{INN for $y$.}\label{fig:inv_inn_1}
    \end{subfigure}
    \begin{subfigure}[t]{0.32\textwidth}
        \includegraphics[width=\textwidth]{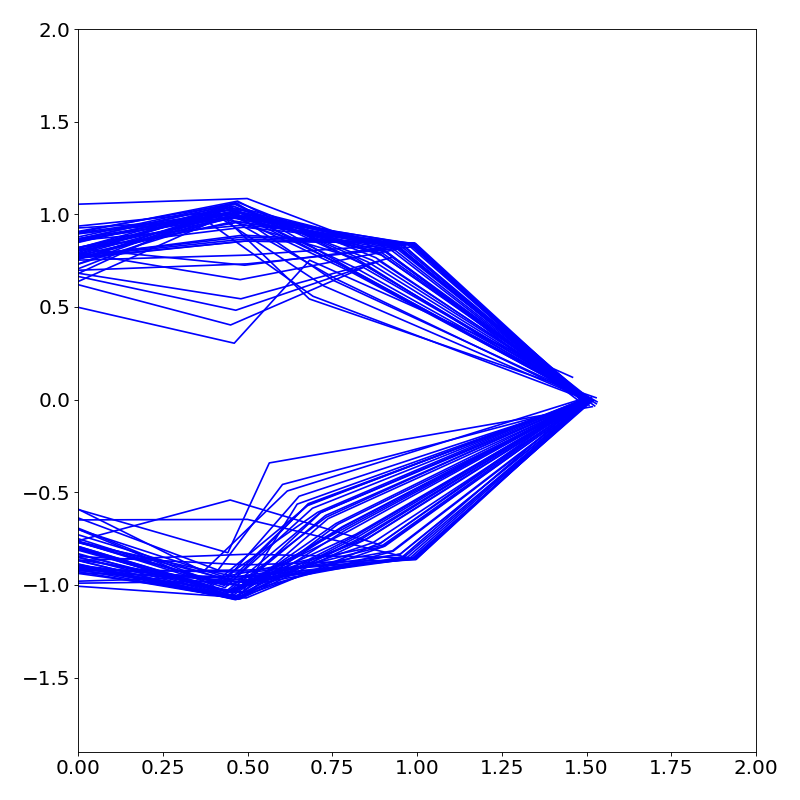}
        \caption{Multimodal INN for $y$.}\label{fig:inv_ginn_1}
    \end{subfigure}
    
    \begin{subfigure}[t]{0.32\textwidth}
        \includegraphics[width=\textwidth]{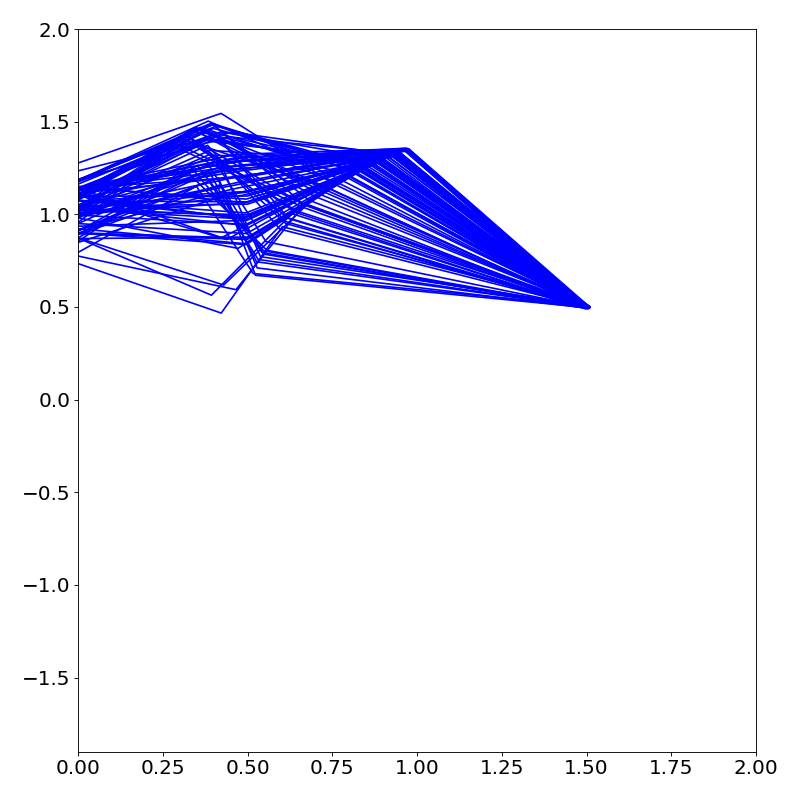}
        \caption{Metropolis--Hastings for $\tilde y$.}\label{fig:rej_2}
    \end{subfigure}
    \begin{subfigure}[t]{0.32\textwidth}
        \includegraphics[width=\textwidth]{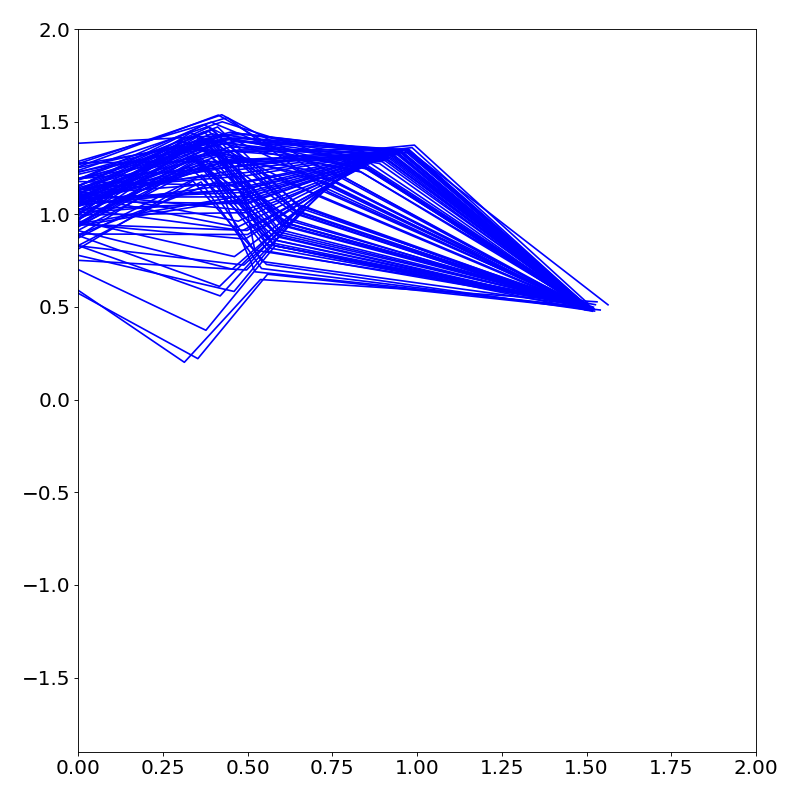}
        \caption{INN for $\tilde y$.}\label{fig:inv_inn_2}
    \end{subfigure}
    \begin{subfigure}[t]{0.32\textwidth}
     \includegraphics[width=\textwidth]{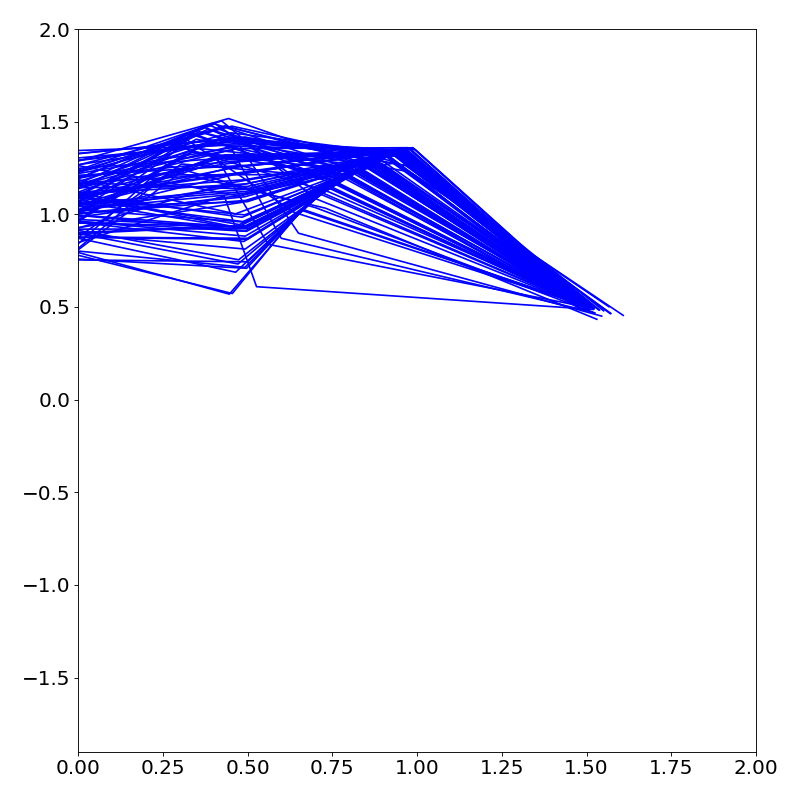}
        \caption{Multimodal INN for $\tilde y$.}\label{fig:inv_ginn_2}
    \end{subfigure}

    \caption{Sampled arm configurations for fixed positions $y$ and $\tilde y$ using different methods.}\label{INN_comp2}
\end{figure}
\begin{table}[t]
\centering
\begin{tabular}{||c | c c c c||} 
 \hline
 Method & $\mathscr{D}_K(x_{i,y},\tilde x_{i,y})$ & $R(y)$ &$\mathscr{D}_K(x_{i, \tilde y},\tilde x_{i,\tilde y})$ & $R(\tilde y)$\\ [0.5ex] 
 \hline\hline
 INN& 0.033 & 0.009 & 0.013 & 0.0003 \\ 
 \hline
Multimodal INN & 0.030 & 0.0003 & 0.010 & 0.0011 \\
\hline
\end{tabular}
\caption{Evaluation of INN and multimodal INN with samples $\tilde x_{i,y}$ created by Alg.~\ref{Algo:ABC}  and samples $x_{i,y}$ from $T^{-1}(y,Z)$.}
\label{tab:inv_kin}
\end{table}}
The predicted posteriors for $y = (0,1.5)$ and $\tilde y = (0.5,1.5)$ are visualized in Fig.~\ref{INN_comp2}, where we plot the four line segments determined by the forward model \eqref{eq:Arm}.
Note that the learned probabilities $p(y) = (0.52,0.48)$ and $p(\tilde y) = (0,1)$ nicely fit the true modality of the data. 
As an approximation quality measure, we first calculate the discrepancy $\mathscr{D}_K(\tilde x_i,x_i)$, where $(\tilde x_i)_{i=1}^{m}$, $m=2000$, are samples generated by Alg.~\ref{Algo:ABC} and $(x_i)_{i=1}^{m}$ are generated by the respective INN methods, i.e., by sampling from $T^{-1}(y,Z)$.
This quantity is then averaged over 5 runs with a single set of rejection samples.
Based on the computed samples and the forward model, we also calculate the resimulation error $R(y) = \frac{1}{m} \sum_i \vert f(x_i)- y\vert^2$.
Both quality measures are computed five times and the average of the obtained results is provided in Tab.~\ref{tab:inv_kin}.
Note that we are evaluating the posterior only for two different $y$.
As the loss function is based on the continuous distribution, results are 
likely to be slightly noisy.

Overall, this experiment confirms that our method excels at modeling multimodal distributions whereas for unimodal a standard INN suffices as well. 
Further, we observe that our method is flexible enough to model both types of distributions even in the case when $y$ is not discrete.

\section{Conclusions}\label{sec:Conclusions}
In this paper, we investigated the latent space of INNs for inverse problems.
Our theoretical results indicate that a proper choice of $Z$ is crucial for obtaining stable INNs with reasonable Lipschitz constants.
To this end, we parameterized the latent distribution $Z$ in dependence on the labels $y$ in order to model properties of the input data.
This resulted in significantly improved behavior of both the INN $T$ and $T^{-1}$.
In some sense, we can interpret this as adding more explicit knowledge to our problem.
Hence, the INN only needs to fit a simpler problem compared to the original setting, which is usually possible with simpler models.

Unfortunately, establishing a result similar to Theorem~\ref{sampl_standard_normal} that also incorporates approximation errors seems quite hard.
A first attempt under rather strong assumptions is given in Theorem~\ref{thm:TV_err}.
Relaxing the assumptions such that they fit to the chosen loss terms would give a useful error estimate between the sampled posterior based on the INN and the true posterior.
Further, it would be interesting to apply our model to real word inverse problems.
Possibly, our proposed framework can be used to automatically detect multimodal structures in such data sets.
To this end, the application of more sophisticated controls on the Lipschitz constants \cite{ComPes20, gouk2018regularisation} could be beneficial.
Additionally, we could also enrich the model by learning the parameters of the permutation matrices using a similar approach as in \cite{HHNP19,glow}.
Finally, it would be interesting to further investigate the role and influence of padding for the training process.

\section*{Acknowledgment}
The authors want to thank Johannes Hertrich and Gabriele Steidl for fruitful discussions on the topic.
\bibliographystyle{abbrv}
\bibliography{sample}

\end{document}